\documentclass[12pt,letterpaper]{article}

\usepackage{microtype}
\usepackage{graphicx}
\usepackage{caption}
\usepackage{subcaption}
\usepackage{booktabs} 
\usepackage[margin=3cm]{geometry}
\usepackage{hyperref}

\usepackage{amsmath}
\usepackage{amssymb}
\usepackage{mathtools}
\usepackage{amsthm}

\usepackage[capitalize,noabbrev]{cleveref}

\theoremstyle{plain}
\newtheorem{theorem}{Theorem}[section]
\newtheorem{proposition}[theorem]{Proposition}
\newtheorem{lemma}[theorem]{Lemma}

\theoremstyle{definition}
\newtheorem{definition}[theorem]{Definition}

\theoremstyle{remark}

\newcommand{\T}{\mathcal{T}}
\DeclareMathOperator{\supp}{supp}
\DeclareMathOperator{\diag}{diag}
\usepackage[textsize=tiny]{todonotes}

\usepackage{authblk}

\begin{document}
\author[1]{Yikun Bai}
\author[1]{Ivan Medri}
\author[2]{Rocio Diaz Martin}
\author[1]{Rana Muhammad Shahroz Khan}
\author[1]{Soheil Kolouri}
\affil[1]{Department of Computer Science, Vanderbilt University}
\affil[2]{Department of Mathematics, Vanderbilt University}

\affil[1]{yikun.bai@vanderbilt.edu}
\affil[1]{ivan.v.medri@vanderbilt.edu}
\affil[2]{rocio.p.diaz.martin@vanderbilt.edu}
\affil[1]{rana.muhammad.shahroz.khan@vanderbilt.edu}
\affil[1]{soheil.kolouri@vanderbilt.edu}
\date{}
\title{Linear optimal partial transport embedding}

\maketitle

\begin{abstract}
Optimal transport (OT) has gained popularity due to its various applications in fields such as machine learning, statistics, and signal processing. However, the balanced mass requirement limits its performance in practical problems. To address these limitations, variants of the OT problem, including unbalanced OT, Optimal partial transport (OPT), and Hellinger Kantorovich (HK), have been proposed. In this paper, we propose the Linear optimal partial transport (LOPT) embedding, which extends the (local) linearization technique on OT and HK to the OPT problem. The proposed embedding allows for faster computation of OPT distance between pairs of positive measures. Besides our theoretical contributions, we demonstrate the LOPT embedding technique in point-cloud interpolation and PCA analysis. Our code is available at \url{https://github.com/Baio0/LinearOPT}.
\end{abstract}

\section{Introduction}
\label{sec:intro}

The Optimal Transport (OT) problem has found numerous applications in machine learning (ML), computer vision, and graphics. The probability metrics and dissimilarity measures emerging from the OT theory, e.g., Wasserstein distances and their variations,  are used in diverse applications, including training generative models \cite{arjovsky2017wasserstein,genevay2017gan,liu2019wasserstein}, domain adaptation \cite{courty2014domain,courty2017joint}, bayesian inference \cite{kim2013efficient}, regression \cite{janati2019wasserstein}, clustering \cite{ye2017fast}, learning from graphs \cite{kolouri2020wasserstein} and point sets \cite{naderializadeh2021pooling,nguyen2023self}, to name a few.  These metrics define a powerful geometry for comparing probability measures with numerous desirable properties, for instance, parameterized geodesics \cite{ambrosio2005gradient}, barycenters \cite{cuturi2014fast}, and a weak Riemannian structure \cite{Villani2003Topics}. 

In large-scale machine learning applications, optimal transport approaches face two main challenges. First, the OT problem is computationally expensive. This has motivated many approximations that lead to significant speedups \cite{cuturi2013sinkhorn,chizat2020faster,scetbon2022lowrank}.  Second, while OT is designed for comparing probability measures, many ML problems require comparing non-negative measures with varying total amounts of mass. This has led to the recent advances in unbalanced optimal transport \cite{Chizat2015Interpolating,chizat2018unbalanced,Liero2018Optimal} and optimal partial transport \cite{caffarelli2010free,figalli2010optimal,figalli2010new}. Such unbalanced/partial optimal transport formulations have been recently used to improve minibatch optimal transport \cite{nguyen2022improving} and for point-cloud registration \cite{bai2022sliced}.

Comparing $K$ (probability) measures requires the pairwise calculation of transport-based distances, which, despite the significant recent computational speed-ups, remains to be relatively expensive. To address this problem, \cite{wang2013linear} proposed the Linear Optimal Transport (LOT) framework, which linearizes the 2-Wasserstein distance utilizing its weak Riemannian structure. In short, the probability measures are embedded into the tangent space at a fixed reference measure (e.g., the measures' Wasserstein barycenter) through a logarithmic map. The Euclidean distances between the embedded measures then approximate the 2-Wasserstein distance between the probability measures. 
 The LOT framework is computationally attractive as it only requires the computation of one optimal transport problem per input measure, reducing the otherwise quadratic cost to linear. Moreover, the framework provides theoretical guarantees on convexifying certain sets of probability measures \cite{moosmuller2020linear,aldroubi2021partitioning}, which is critical in supervised and unsupervised learning from sets of probability measures. The LOT embedding has recently found diverse applications, from comparing collider events in physics \cite{cai2020linearized} and comparing medical images \cite{basu2014detecting,kundu2018discovery} to permutation invariant pooling for comparing graphs \cite{kolouri2020wasserstein} and point sets \cite{naderializadeh2021pooling}.

Many applications in ML involve comparing non-negative measures (often empirical measures) with varying total amounts of mass, e.g., domain adaptation \cite{fatras2021unbalanced}. Moreover, OT distances (or dissimilarity measures) are often not robust against outliers and noise, resulting in potentially high transportation costs for outliers. Many recent publications have focused on variants of the OT problem that allow for comparing non-negative measures with unequal mass. For instance, the optimal partial transport (OPT) problem \cite{caffarelli2010free,figalli2010optimal,figalli2010new}, 
 Kantorovich--Rubinstein norm \cite{guittet2002extended,lellmann2014imaging}, and the Hellinger--Kantorovich distance \cite{chizat2018interpolating,Liero2018Optimal}. These methods fall under the broad category of ``unbalanced optimal transport'' \cite{chizat2018unbalanced,Liero2018Optimal}. The existing solvers for ``unbalanced optimal transport'' problems are generally as expensive or more expensive than the OT solvers. Hence, computation time remains a main bottleneck of these approaches.

 To reduce the computational burden for comparing unbalanced measures, \cite{cai2022linearized} proposed a clever extension for the LOT framework to unbalanced nonnegative measures by linearizing the Hellinger-Kantorovich, denoted as Linearized Hellinger-Kantorovich (LHK), distance, with many desirable theoretical properties. However, an unintuitive caveat about HK and LHK formulation is that the geodesic for the transported portion of the mass does not resemble the OT geodesic. In particular, the transported mass does not maintain a constant mass as it is transported (please see Figure \ref{fig: HK and OPT}). In contrast, OPT behaves exactly like OT for the transported mass with the trade-off of losing the Riemannian structure of HK.

\textbf{Contributions:} In this paper, inspired by OT geodesics, we provide an OPT interpolation technique using its dynamic formulation and explain how to compute it for empirical distributions using barycentric projections. We use this interpolation to embed the space of measures in a euclidean space using optimal partial transport concerning a reference measure. This allows us to extend the LOT framework to LOPT, a linearized version of OPT. Thus, we reduce the computational burden of OPT while maintaining the decoupling properties between noise (created and destroyed mass) and signal (transported mass) of OPT. We propose a LOPT discrepancy measure and a LOPT interpolating curve and contrast them with their OPT counterparts. Finally, we demonstrate applications of the new framework in point cloud interpolation and PCA analysis, showing that the new technique is more robust to noise.

\textbf{Organization:} In section \ref{sec: background}, we review Optimal Transport Theory and the Linear Optimal Transport framework to set the basis and intuitions on which we build our new techniques. In Section \ref{sec: LOPT} we review Optimal Partial Transport Theory and present an explicit solution to its Dynamic formulation that we use to introduce the Linear Optimal Partial Transport framework (LOPT). We define LOPT Embedding, LOPT Discrepancy, LOPT interpolation and give explicit ways to work with empirical data. In Section \ref{sec: applications} we show  applications of the LOPT framework to approximate OPT distances, to interpolate between point cloud datasets, and to preprocess data for PCA analysis. In the appendix, we provide proofs for all the results, new or old, for which we could not find a proof in the literature.




\begin{figure}
\centering
\includegraphics[width=\textwidth]{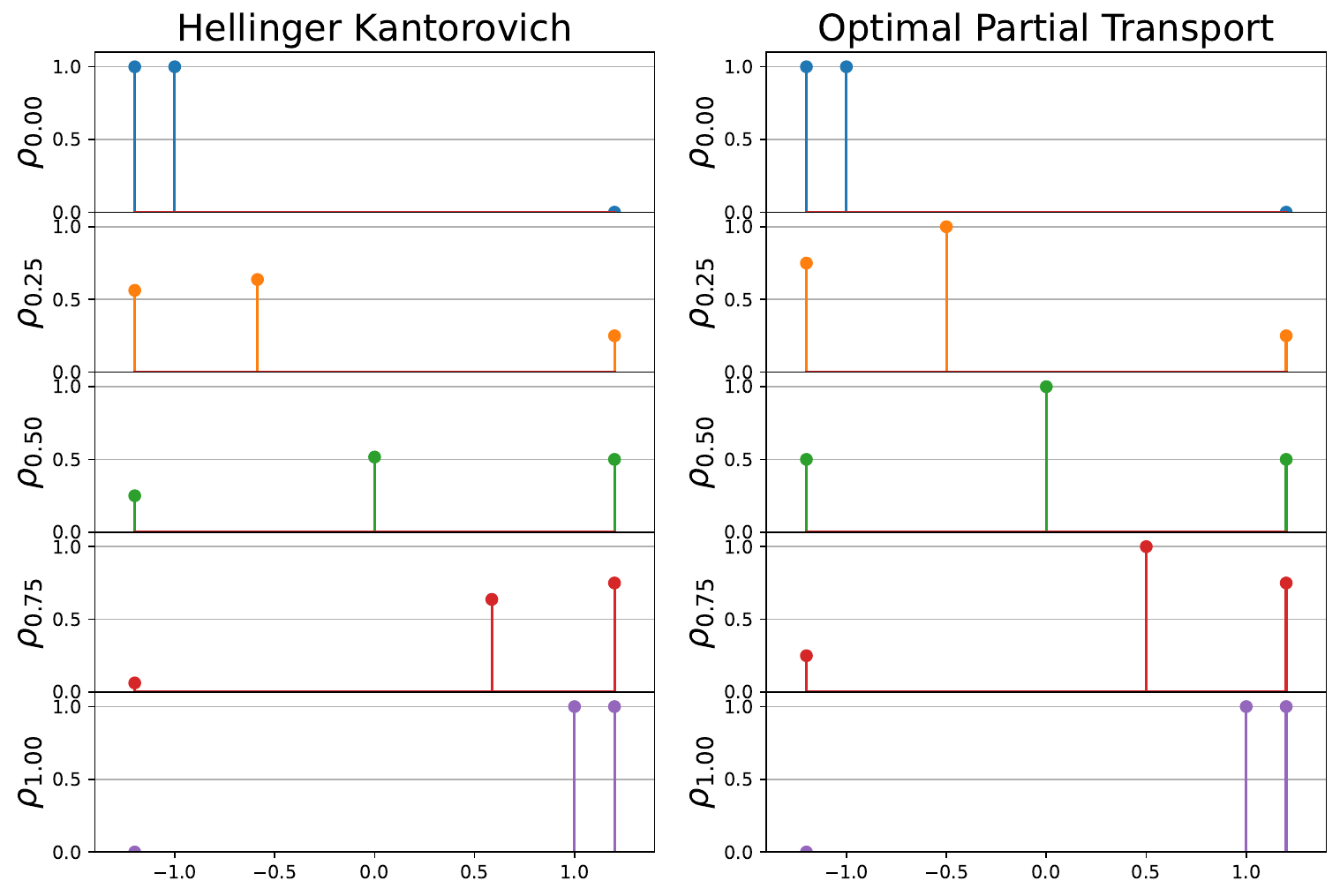}
\caption{The depiction of the HK and OPT geodesics between two measures, at times $t\in\{0,0.25,0.5,0.75,1\}$. The top row (Blue) represents two initial deltas of mass one located at positions -1.2 and -1. The bottom row (Purple) shows two final deltas of mass one located at 1 and 1.2. At intermediate time steps $t=0.25,0.5,0.75$, the transported part (middle delta moving from -1 to 1) changes mass for HK while its mass remains constant for OPT. Outer masses (located at -1.2 for initial time $t=0$, and at 1.2 for final time $t=1$) are being destroyed and created, so mass changes are expected. Notably, mass is created/destroyed with a linear rate for OPT and a nonlinear rate for HK.
See Appendix \ref{sec: HK vs OPT} for further analysis.}
\label{fig: HK and OPT}
\end{figure}

\section{Background: OT and LOT}\label{sec: background}

\subsection{Static Formulation of Optimal Transport}
Let $\mathcal{P}(\Omega)$ be the set of Borel probability measures defined in a convex compact subset $\Omega$ of $\mathbb{R}^d$, and consider $\mu^0,\mu^j\in \mathcal{P}(\Omega)$. 
The Optimal Transport (OT) problem between $\mu^0$ and $\mu^j$ is that of finding the cheapest way to transport $\textit{all}$ the mass distributed according to the \textit{reference} measure $\mu^0$ onto a new distribution of mass determined by the \textit{target} measure $\mu^j$.
Mathematically, it was stated by Kantorovich as the minimization problem
\begin{align}
    &OT(\mu^0,\mu^j):= \inf_{\gamma\in\Gamma(\mu^0,\mu^j)}C(\gamma;\mu^0,\mu^j) \label{eq: OT} \\
    &\text{for } \quad C(\gamma;\mu^0,\mu^j) := \int_{\Omega^2} \|x^0-x^j\|^2 d\gamma(x^0,x^j)\label{eq: OT cost},
\end{align}
where $\Gamma(\mu^0,\mu^j)$ is the set of all joint probability measures in $\Omega^2$ with marginals $\mu^0$ and $\mu^j$. A measure $\gamma\in\Gamma(\mu^0,\mu^j)$ is called a \textit{transportation plan}, and given
measurable sets $A,B \in\Omega$, the coupling $\gamma(A\times B)$ describes how much
mass originally in the set $A$ is transported into
the set $B$. The squared of the Euclidean distance\footnote{More general cost functions might be used, but they are beyond the scope of this article.} $\|x^0-x^j\|^2$ is interpreted as the cost of transporting a unit mass located at $x^0$ to $x^j$.  Therefore,   $C(\gamma;\mu^0,\mu^j)$ represents the total cost of moving $\mu^0$ to $\mu^j$ according to $\gamma$. Finally, we will denote the set of all plans that achieve the infimum in \eqref{eq: OT}, which is non-empty \cite{Villani2003Topics},  as $\Gamma^*(\mu^0,\mu^j)$.

Under certain conditions (e.g. when $\mu^0$ has continuous density), an optimal plan $\gamma$ can be induced by a rule/map $T$ that takes all the mass at each position $x$ to a unique point $T(x)$. 
If that is the case, we say that $\gamma$ 
does not split mass and that it is \textbf{induced by a map T}. In fact, it is concentrated on the graph of $T$ in the sense that for all measurable sets $A,B\subset\Omega$,
$ \, \gamma(A\times B)=\mu^0(\{x\in A: \, T(x)\in B\})$, and we will write it as the \textit{pushforward} $\gamma=(\mathrm{id}\times T)_\#\mu^0$.
Hence,  \eqref{eq: OT} reads as
\begin{equation}
    OT(\mu^0,\mu^j) = \int_{\Omega} \|x-T(x)\|^2 d\mu^0(x) \label{eq: OT map cost}
\end{equation}
The function $T:\Omega\to\Omega$ is called a \textit{Monge map}, and when $\mu^0$ is absolutely continuous it is unique \cite{brenier1991polar}. 

Finally, the square root of the optimal value $OT(\cdot,\cdot)$ is exactly the so-called \textbf{Wasserstein distance}, $W_2$, in $\mathcal{P}(\Omega)$ \cite[Th.7.3]{Villani2003Topics}, and we will call it also as \textbf{OT squared distance}. In addition, with this distance,  
$\mathcal{P}(\Omega)$ is not only a metric space but also a Riemannian manifold \cite{Villani2003Topics}. In particular, the tangent space of any
$\mu\in \mathcal{P}(\Omega)$ is $\T_{\mu}=L^2(\Omega;\mathbb{R}^d,\mu)=\{u:\Omega\to \mathbb{R}^d: \|u\|_{\mu}^2<\infty\}$, where 
\begin{equation}\label{eq: norm tangent space}
  \|u\|_{\mu}^2:=\int_\Omega \|u(x)\|^2d\mu(x).  
\end{equation}

\subsection{Dynamic Formulation of Optimal Transport}\label{sec: ot dynamic}
To understand the framework of Linear Optimal Transport (LOT) we will use the dynamic formulation of the OT problem. Optimal plans and maps can be viewed as a static way of matching two distributions. They tell us where each mass in the initial distribution should end, but they do not tell the full story of how the system evolves from initial to final configurations. 


In the dynamic formulation, we consider $\rho\in\mathcal{P}([0,1]\times\Omega)$ a curve of measures parametrized in time that describes the distribution of mass $\rho_t:=\rho(t,\cdot)\in\mathcal{P}(\Omega)$ at each instant $0 \leq t\leq 1$. We will require the curve to be sufficiently smooth, to have boundary conditions $\rho_0 = \mu^0$, $\rho_1 = \mu^j$, and to satisfy the conservation of mass law. Then, it is well known that there exists a velocity vector field $v_t:=v(t,\cdot)$ such that $\rho_t$ satisfies the continuity equation\footnote{The continuity equation is satisfied weakly or in the sense of distributions. See \cite{Villani2003Topics,sant}.} with boundary conditions 
\begin{equation} \label{eq: continuity eq}
    \partial_t\rho+\nabla\cdot \rho v=0, \qquad 
    \rho_0=\mu^0,\quad \rho_1=\mu^j.    
\end{equation} 
The length\footnote{Precisely, the length of the curve $\rho$, with respect to the Wasserstein distance, should be $\int_0^1 \|v_t\|_{\rho_t}dt$, but this will make no difference in the solutions of \eqref{eq: OT dynamic} since they are constant speed geodesics. 
} of the curve can be stated as
$\int_{[0, 1]\times\Omega}\|v\|^2 d  \rho:=\int_0^1 \|v_t\|_{\rho_t}^2dt$, 
for $\|\cdot\|_{\rho_t}$ as in \eqref{eq: norm tangent space}, and 
$OT(\mu^0,\mu^j)$ coincides with the length of the shortest curve between $\mu^0$ and $\mu^j$ \cite{benamou2000computational}. Hence, the dynamical formulation of the OT problem \eqref{eq: OT} reads as
\begin{align}
    OT(\mu^0,\mu^j)=\inf_{(\rho,v)\in\mathcal{CE}(\mu^0,\mu^j)}\int_{[0, 1]\times\Omega}\|v\|^2d\rho,\label{eq: OT dynamic}    
\end{align}
where $\mathcal{CE}(\mu^0,\mu^j)$ is the set of pairs $(\rho,v)$, where $\rho \in \mathcal{P}([0,1]\times\Omega)$, and $v:[0,1]\times\Omega\to\mathbb{R}^d$,  satisfying \eqref{eq: continuity eq}.

Under the assumption of existence of an optimal Monge map $T$, an optimal solution for \eqref{eq: OT dynamic} can be given explicitly and is pretty intuitive. If a particle starts at position $x$ and finishes at position $T(x)$, then for $0<t<1$ it will
be at the point
\begin{equation}\label{eq: ot X}
    T_t(x):=(1-t)x +tT(x).
\end{equation}
Then, varying both the time $t$ and $x\in\Omega$, the mapping \eqref{eq: ot X} can be interpreted as a flow whose time velocity\footnote{For each $(t,x)\in(0,1)\times \Omega$, the vector $v_t(x)$ is well defined as $T_t$ is invertible. See \cite[Lemma 5.29]{sant}.} is
\begin{equation}
    v_t(x) = T(x_0)-x_0,  \qquad  \text{ for } x=T_t(x_0).\label{eq: ot v}  
\end{equation}
To obtain the curve of probability measures $\rho_t$, one can evolve $\mu^0$ through the flow $T_t$ using the formula $\rho_t(A) = \mu_0(T_t^{-1}(A))$ for any measurable set $A$. 
That is, $\rho_t$ is the \textit{push-forward} of $\mu^0$ by $T_t$
\begin{equation}\label{eq: ot rho} 
     \rho_t=(T_t)_\#\mu^0, \qquad 0\leq t\leq 1.    
\end{equation}
The pair $(\rho,v)$ defined by \eqref{eq: ot rho} and \eqref{eq: ot v} satisfies the continuity equation \eqref{eq: continuity eq} and solves \eqref{eq: OT dynamic}. Moreover, the curve $\rho_t$ is a \textit{constant speed geodesic} in $\mathcal{P}(\Omega)$ between $\mu^0$ and $\mu^j$   \cite{figalli2021invitation}, i.e., it satisfies that for all $0\leq s\leq t\leq 1$
\begin{equation} \label{eq: OT geodesic cond}
    \sqrt{OT(\rho_s,\rho_t)}=(t-s)\sqrt{OT(\rho_0,\rho_1)}. 
\end{equation} 
A confirmation of this comes from comparing  the OT cost \eqref{eq: OT map cost} with \eqref{eq: ot v} obtaining
\begin{align}
    OT(\mu^0,\mu^j)&=\int_\Omega \|v_0(x)\|^2d\mu^0(x)\label{eq: ot dynamic solution} 
\end{align}
which tells us that we only need the speed at the initial time to compute the total length of the curve. Moreover, $OT(\mu^0,\mu^j)$ coincides with the squared norm of the tangent vector $v_0$ in the tangent space $\mathcal{T}_{\mu^0}$ of $\mathcal{P}(\Omega)$ at $\mu^0$. 

\subsection{Linear Optimal Transport Embedding}\label{sec: LOT embedding}
Inspired by the induced Riemannian geometry of the $OT$ squared distance, \cite{wang2013linear} proposed the so-called \textbf{Linear Optimal Transportation} (LOT) framework. Given two target measures $\mu^i,\mu^j$, the main idea relies on considering a reference measure $\mu^0$ 
and  embed these target measures into the tangent space $\T_{\mu^0}$. This is done by identifying each measure $\mu^j$ with the curve \eqref{eq: ot rho} minimizing $OT(\mu^0,\mu^j)$ and computing its velocity (tangent vector) at $t=0$ using \eqref{eq: ot v}. 

Formally, let us fix a continuous probability reference measure $\mu^0$. Then, the \textbf{LOT embedding} \cite{moosmuller2020linear} is defined as
\begin{align}
\mu^j\mapsto u^j:=T^j-\mathrm{id}  \qquad \forall \  \mu^j\in \mathcal{P}(\Omega)\label{eq: linear OT embedding continuous}
\end{align}
where $T^j$ is the optimal Monge map between $\mu^0$ and $\mu^j$. Notice that by \eqref{eq: OT map cost},  \eqref{eq: norm tangent space} and \eqref{eq: ot dynamic solution} we have
\begin{equation}\label{eq: LOT norm}
    \|u^j\|^2_{\mu^0}=OT(\mu^0,\mu^j).
\end{equation}
After this embedding, one can use the distance in $\T_{\mu^0}$ 
between the projected measures to define a new distance in $\mathcal{P}(\Omega)$ that can be used to approximate $OT(\mu^i,\mu^j)$.
The \textbf{LOT squared distance}  is defined as
\begin{align}
    LOT_{\mu^0}(\mu^i,\mu^j):=\|u^i-u^j\|_{\mu^0}^2. \label{eq: LOT continuous} 
\end{align}

\subsection{LOT in the Discrete Setting}\label{subsec: LOT Discrete}
For discrete probability measures $\mu^0,\mu^j$ of the form
\begin{equation}\label{eq: discrete measures}
    \mu^0=\sum_{n=1}^{N_{0}}p^0_n\delta_{x^0_n}, \qquad \mu^j=\sum_{n=1}^{N_{j}}p^j_n\delta_{x^j_n},
\end{equation}
a Monge map $T^j$  for $OT(\mu^0,\mu^j)$ may not exist.
Following \cite{wang2013linear}, in this setting, the target measure $\mu^j$ can be replaced by a new measure $\hat\mu^j$ for which an optimal transport Monge map exists.
For that, given an optimal plan $\gamma^j\in\Gamma^*(\mu^0,\mu^j)$, it can be viewed as a $N_0\times N_i$ matrix whose value at position $(n,m)$ represents how much mass from $x^0_n$  should be taken to $x^j_m$.
Then, we define the \textbf{OT barycentric projection}\footnote{We refer to \cite{ambrosio2005gradient} for the rigorous definition.} of $\mu^j$ \textbf{with respect to $\mu^0$} as 
\begin{equation}
    \hat{\mu}^j:=\sum_{n=1}^{N_0}p^0_n\delta_{\hat{x}_n^j}, \, \text{where} \quad 
    \hat{x}_n^j:=\frac{1}{p_n^0}\sum_{m=1}^{N_j}\gamma^j_{n,m}x_m^j. \label{eq: discrete barycenter}
\end{equation}
The new measure $\hat\mu^j$ is regarded as an $N_0$-point representation of the target measure $\mu^j$.
The following lemma guarantees the existence of a Monge map between $\mu^0$ and $\hat{\mu}^j$.
\begin{lemma}\label{lem: ot barycentric projection} 
    Let  $\mu^0$ and $\mu^j$ be two discrete probability measures as in \eqref{eq: discrete measures}, and consider an OT barycentric projection $\hat{\mu}^j$ of $\mu^j$ with respect to $\mu^0$ as in \eqref{eq: discrete barycenter}. Then, the map
    $x_n^0\mapsto \hat{x}_n^j$ given by \eqref{eq: discrete barycenter} solves the OT problem $OT(\mu^0,\hat\mu^j)$.
\end{lemma}
It is easy to show that if the optimal transport plan $\gamma^j$ is induced by a Monge map, then $\hat{\mu}^j = \mu^j$. As a consequence, the OT barycentric projection is an actual projection in the sense that it is idempotent. 

Similar to the continuous case \eqref{eq: linear OT embedding continuous}, given a discrete reference measure $\mu^0$, we can define the \textbf{LOT embedding} for 
a discrete measure $\mu^j$ as the rule
\begin{equation}
    \mu^j\mapsto u^{j}:=[(\hat x_1^j-x^0_1),\ldots ,(\hat x_{N_0}^j- x^0_{N_0})]
    \label{eq: LOT embedding}.\footnote{
    In \cite{wang2013linear}, the embedded vector is defined as the the element-wise multiplication $u^j\bigodot \sqrt{p^0}$,  In addition, $\hat{\mu}^j, u^j$ are determined by $\gamma^j$, but we ignore the subscript `$\gamma^j$' for convenience.
    }
\end{equation}
The range $\mathcal{T}_{\mu_0}$ of this application is identified with $\mathbb{R}^{d\times N_0}$ with the norm $\|u\|_{\mu^0} := \sum_{n=1}^{N_0}\|u(n)\|^2p_n^0$, where $u(n)\in\mathbb{R}^d$ denotes the $n$th entry of $u$. We call $(\mathbb{R}^{d\times N_0}, \|\cdot\|_{\mu^0})$ the embedding space.

By the discussion above, if the optimal plan $\gamma^j$ for problem $\text{OT}(\mu^0,\mu^j)$ is induced by a Monge map, then the discrete embedding is consistent with \eqref{eq: LOT norm} in the sense that
\begin{align}
    \|u^j\|_{\mu^0}^2= OT(\mu^0,\hat\mu^j) = OT(\mu^0,\mu^j) \label{eq: OT and embedding}.
\end{align}
Hence, as in section \ref{sec: LOT embedding}, we can use the distance between embedded measures in ($\mathbb{R}^{d\times{N_0}},\|\cdot\|_{\mu_0})$  to define a \textit{discrepancy} in the space of discrete probabilities that can be used to approximate $OT(\mu^i,\mu^j)$. The 
\textbf{LOT discrepancy}\footnote{In \cite{wang2013linear}, LOT is defined by the infimum over all possible optimal pairs $(\gamma^i,\gamma^j)$. We do not distinguish these two formulations for convenience in this paper. Additionally, \eqref{eq: discrete LOT} is determined by the choice of $(\gamma^i,\gamma^j)$.} is defined as 
\begin{align}
    LOT_{\mu^0}(\mu^i,\mu^j):=\|u^{i}-u^{j}\|_{\mu^0}^2 \label{eq: discrete LOT}.
\end{align}
We call it a \textit{discrepancy} because it is not a 
squared metric between discrete measures. It does not necessarily satisfy that $LOT(\mu^i,\mu^j)\not=0$ for every distinct $\mu^i,\mu^j$. Nevertheless, $\|u^i-u^j\|_{\mu^0}$ is a metric in the embedding space.

\subsection{OT and LOT Geodesics in Discrete Settings}\label{sec: OT LOT geogesic}
Let $\mu^i$, $\mu^j$ be discrete probability measures as in \eqref{eq: discrete measures} (with `$i$' in place of $0$). If an optimal Monge map $T$ for $OT(\mu^i,\mu^j)$ exists, a constant speed geodesic $\rho_t$ between $\mu^i$ and $\mu^j$, for the OT squared distance, can be found by mimicking \eqref{eq: ot rho}. Explicitly, with $T_t$ as in \eqref{eq: ot X},
\begin{equation}
 \rho_t=(T_t)_\#\mu^i=\sum_{n=1}^{N_i}p_n^i\delta_{(1-t)x_n^i+t T(x_n^i)}. \label{eq: ot geodesic discrete general}
\end{equation}
In practice, one replaces $\mu^j$ by its OT barycentric projection with respect to $\mu^i$ (and so, the existence of an optimal Monge map is guaranteed by Lemma \ref{lem: ot barycentric projection}).

Now, given a discrete reference $\mu^0$, the LOT discrepancy provides a new structure to the space of discrete probability densities. Therefore, we can provide a substitute for the OT geodesic \eqref{eq: ot geodesic discrete general} between $\mu^i$ and $\mu^j$. 
Assume we have the embeddings $\mu^i\mapsto u^i$, $\mu^j\mapsto u^j$ as in \eqref{eq: LOT embedding}. The geodesic between  $u^i$ and $u^j$ in  the LOT embedding space $\mathbb{R}^{d\times N_0}$ has the simple form $u_t = (1-t)u^i + tu^j$. This correlates with the curve $\hat\rho_t$ in $\mathcal{P}(\Omega)$  induced by the map $\hat T:\hat x_n^i\mapsto \hat x_n^j$\ \footnote{This map can be understood as the one that transports $\hat\mu^i$ onto $\hat\mu^j$ pivoting on the reference:  $\hat\mu^i\mapsto \mu^0\mapsto \hat\mu^j$.} as
\begin{equation}
    \hat\rho_t:=(\hat T_t)_\#\hat{\mu}^i=\sum_{n=1}^{N_0} p_n^0\delta_{{x}_n^0 +u_t(n)}.\label{eq: lot geodesic}
\end{equation}
By abuse of notation, we call this curve the \textbf{LOT geodesic} between $\mu^i$ and $\mu^j$. Nevertheless, it is a \textit{geodesic between their barycentric projections} since it satisfies the following.
\begin{proposition}\label{thm: lot geodesic}
    Let $\hat\rho_t$ be defined as \eqref{eq: lot geodesic}, and $\hat\rho_0=\hat\mu^i,\hat\rho_1=\hat\mu^j$, then for all $0\leq s\leq t\leq 1$
    \begin{equation}
        \sqrt{LOT_{\mu^0}(\hat\rho_s,\hat\rho_t)}=(t-s)\sqrt{LOT_{\mu^0}(\hat\rho_0,\hat\rho_1)} \label{eq: lot geodesic cond}.
    \end{equation}
\end{proposition}

\section{Linear Optimal Partial Transport Embedding}\label{sec: LOPT}

\subsection{Static Formulation of Optimal Partial Transport}
In addition to mass transportation, the OPT problem allows mass destruction at the source and mass creation at the target.
 Let $\mathcal{M}_+(\Omega)$ denote the set of all positive finite Borel measures defined on $\Omega$. 
 For $\lambda\geq 0$ the OPT problem between $\mu^0,\mu^j\in\mathcal{M}_+(\Omega)$ can be formulated as
\begin{align}
    &OPT_{\lambda}(\mu^0,\mu^j):=\inf_{\gamma\in\Gamma_{\leq}(\mu^0,\mu^j)} 
    C(\gamma;\mu^0,\mu^j,\lambda) \label{eq: OPT}\\
    &\text{for} \quad      
    C(\gamma;\mu^0,\mu^j,\lambda):=\int_{\Omega^2} \|x^0-x^j\|^2 d\gamma(x^0,x^j)+\lambda(|\mu^0-\gamma_0|+|\mu^j-\gamma_1|) \label{eq: OPT cost}
\end{align}
where $|\mu^0-\gamma_0|$ is the total mass of $\mu^0-\gamma_0$ (resp. $|\mu^j-\gamma_1|$), and  $\Gamma_{\leq}(\mu^0,\mu^j)$ denotes the set of all measures in $\Omega^2$ with marginals $\gamma_0$ and $\gamma_1$ satisfying  $\gamma_0\leq \mu^0$ (i.e.,  $\gamma_0(E)\leq\mu^0(E)$ for all measurable set $E$), and $\gamma_1\leq \mu^j$.
Here, the mass destruction and creation penalty is linear, parametrized by $\lambda$. The set of minimizers $\Gamma_\leq^*(\mu^0,\mu^j)$ of  \eqref{eq: OPT} is non-empty \cite{figalli2010optimal}. One can further restrict $\Gamma_\leq(\mu^0,\mu^j)$ to the set of partial transport plans $\gamma$ such that $ \|x^0-x^j\|^2<2\lambda$ for all $ (x^0,x^j)\in \supp(\gamma)$  \cite[Lemma 3.2]{bai2022sliced}. This means that if the usual transportation cost is greater than $2\lambda$, it is better to create/destroy mass.

\subsection{Dynamic Formulation of Optimal Partial Transport}
Adding a forcing term $\zeta$ to the continuity equation \eqref{eq: OT dynamic}, one can take into account curves that allow creation and destruction of mass. That is, those who break the conservation of mass law.  Thus, it is natural that the minimization problem \eqref{eq: OPT} can be rewritten \cite[Th. 5.2]{chizat2018unbalanced} into a dynamic formulation as
\begin{equation}
    OPT_\lambda(\mu^0,\mu^j)=\inf_{(\rho,v,\zeta)\in\mathcal{FCE}(\mu^0,\mu^j)}\int_{[0,1]\times\Omega}\|v\|^2d\rho +\lambda |\zeta|
    \label{eq: opt dynamic}
\end{equation}
where $\mathcal{FCE}(\mu^0,\mu^j)$ is the set of  tuples $(\rho,v,\zeta)$ such that $\rho \in \mathcal{M}_+([0,1]\times\Omega)$, $\zeta\in \mathcal{M}([0,1]\times\Omega)$ (where $\mathcal{M}$ stands for signed measures) and $v:[0,1]\times\Omega\to\mathbb{R}^d$,  satisfying
\begin{equation}\label{eq: continuity equation opt}
    \partial_t\rho+\nabla\cdot \rho v=\zeta , \qquad
    \rho_0=\mu^0,\quad \rho_1=\mu^j.   
\end{equation}
As in the case of OT, under certain conditions on the minimizers $\gamma$ of \eqref{eq: OPT}, one curve $\rho_t$ that minimizes the dynamic formulation \eqref{eq: opt dynamic} is quite intuitive. We show in the next proposition that it consists of three parts $\gamma_t$, $(1-t)\nu_0$ and $t\nu^j$ (see \eqref{eq: nu}, \eqref{eq: gamma_t}, and \eqref{eq: opt rho} below). The first is a curve that only transports mass, and the second and third destroy and create mass at constant rates $|\nu_0|$, $|\nu^j|$, respectively. 
\begin{proposition}\label{pro: opt dynamic}
    Let $\gamma^*\in\Gamma_\leq^*(\mu^0,\mu^j)$ be of the form
    $\gamma^*=(\mathrm{id}\times T)_\#\gamma_0^* \, $ for $T: \Omega\to\Omega$ a (measurable) map. 
    Let 
    \begin{align}
        &\nu_0 := \mu^0-\gamma_0^*, \quad \nu^j := \mu^j-\gamma_1^*, \label{eq: nu}\\
        &T_t(x):=(1-t)x+tT(x), \quad \gamma_t:= (T_t)_\#\gamma_0^*. \label{eq: gamma_t}
    \end{align}
    Then, an optimal solution $(\rho,v,\zeta)$ for \eqref{eq: opt dynamic} is given by  
    {\begin{align}       
        &\rho_t:=\gamma_t+(1-t)\nu_0+t\nu^j \label{eq: opt rho}, \\
        &v_t(x):= T(x_0)-x_0,
            \qquad \text{ if } x=T_t(x_0).
        \label{eq: opt v}\\         &\zeta_t:=\nu^j-\nu_0 \label{eq: opt zeta}
    \end{align}}
     Moreover,  plugging in $(\rho,v,\zeta)$ into \eqref{eq: opt dynamic},         
        it holds that
        \begin{equation}
        {OPT}_{\lambda}(\mu^0,\mu^j)=\|v_0\|^2_{\gamma_0^*,2\lambda}+\lambda(|\nu_0|+|\nu^j|), \label{eq: opt dynamic solution}    
        \end{equation}
        where $v_0(x)=T(x)-x$  (i.e., $v_t$ at time $t=0$), and 
$$\|v\|_{\mu,2\lambda}^2:=\int_\Omega \min(\|v\|^2,2\lambda)d\mu, \qquad \text{for } v:\Omega\to\mathbb{R}^d.$$
\end{proposition}
 In analogy to the OT squared distance, we also call the optimal partial cost  \eqref{eq: opt dynamic solution} as the \textbf{OPT squared distance}.

\subsection{Linear Optimal Partial Transport Embedding}
\begin{definition}
    Let $\mu^0$, $\mu^j\in\mathcal{M}_+(\Omega)$ such that $OPT_\lambda(\mu^0,\mu^j)$ is solved by a plan induced by a map. The \textbf{LOPT embedding} of $\mu^j$ with respect to $\mu^0$ is defined as
    \begin{equation}
        \mu^j\mapsto (u^j,\bar\mu^j,\nu^j):=(v_0,\gamma_0,\nu^j) \label{eq: lopt embedding general}  
    \end{equation}
    where $v_0,\gamma_0,\nu^j$ are defined as in Proposition \ref{pro: opt dynamic}. 
\end{definition}

Let us compare the LOPT \eqref{eq: lopt embedding general} and LOT \eqref{eq: linear OT embedding continuous} embeddings. The first component $v_0$ represents the tangent of the curve that transports mass from the reference to the target. This is exactly the same as the LOT embedding. In contrast to LOT, the second component $\gamma_0$ is necessary since we need to specify what part of the reference is being transported. The third component $\nu^j$ can be thought of as the tangent vector of the part that creates mass. There is no need to save the destroyed mass because it can be inferred from the other quantities.

Now, let $\mu^0\wedge \mu^j$ be the \textit{minimum measure}\footnote{Formally,    ${\mu}^0\wedge {\mu}^j(B):=\inf \left\{{\mu}^0\left(B_1\right)+{\mu}^j\left(B_2\right)\right\}$ for every Borel set $B$, where the infimum is taken over all partitions of $B$., i.e. $B=B_1 \cup B_2$, $B_1 \cap B_2=\emptyset$, given by Borel sets $B_1$, $B_2$.} between $\mu^0$ and $\mu^j$.
By the above definition, $\mu^0\mapsto (u^0,\bar{\mu}^0,\nu^0)=(0,\mu^0,0)$. 
Therefore, \eqref{eq: opt dynamic solution} can be rewritten 
\begin{align}
OPT_\lambda(\mu^0,\mu^j)&=\|u^0-u^j\|^2_{\bar{\mu}^0\wedge \bar{\mu}^j,2\lambda} 
+\lambda(|\bar\mu^0-\bar\mu^j|+|\nu_0|+|\nu^j|)\label{eq: opt and opt embedding} 
\end{align}
This motivates the definition of the \textbf{LOPT discrepancy}.\footnote{$LOPT_\lambda$ is not a rigorous metric.}
\begin{definition}\label{def: discrete opt embedding} 
    Consider a reference $\mu^0\in \mathcal{M}_+(\Omega)$ and target measures $\mu^i,\mu^j\in \mathcal{M}_+(\Omega)$ such that  $OPT_\lambda(\mu^0,\mu^i)$ and $ OPT_\lambda(\mu^0,\mu^j)$ can be solved by plans induced by mappings as in the hypothesis of Proposition \ref{pro: opt dynamic}. Let $(u^i,\bar\mu^i,\nu^i)$ and $(u^j,\bar\mu^j,\nu^j)$ be the LOPT embeddings of $\mu^i$ and $\mu^j$ with respect to $\mu^0$. The  \textbf{LOPT discrepancy} between $\mu^i$ and $\mu^j$ with respect to $\mu^0$ is defined as 
    \begin{align}
        LOPT_{\mu^0,\lambda}(\mu^i,\mu^j)&:=\|u^i-u^j\|^2_{\bar\mu^i\wedge \bar\mu^j,2\lambda}+\lambda(|\bar{\mu}^i-\bar{\mu}^j|+|\nu^i|+|\nu^j|)\label{eq: LOPT general}
    \end{align}
\end{definition}
Similar to the LOT framework, by equation \eqref{eq: opt and opt embedding}, LOPT can recover OPT when $\mu^i=\mu^0$. That is, 
\begin{equation*}
    LOPT_{\mu^0,\lambda}(\mu^0,\mu^j)=OPT_\lambda(\mu^0,\mu^j).
\end{equation*}

\subsection{LOPT in the Discrete Setting} 
If $\mu^0, \mu^j$ are $N_0, N_j-$size discrete non-negative measures as in \eqref{eq: discrete measures} (but not necessarily with total mass 1), the OPT problem \eqref{eq: OPT} can be written as
    \begin{equation*}
        \min_{\gamma\in\Gamma_{\leq}(\mu^0,\mu^j)} \sum_{n,m} \|x_{n}^0-x_{m}^j\|^2\gamma_{n,m}+\lambda(|p^0|+|p^j|-2|\gamma|) \label{eq: OPT empirical}
    \end{equation*}
where the set $\Gamma_{\leq}(\mu^0,\mu^j)$ can be viewed as the subset of $N_0\times N_j$ matrices with non-negative entries
\[ \Gamma_{\leq}(\mu^0,\mu^j):= \{\gamma\in \mathbb{R}_+^{N_0\times {N_j}}: \gamma 1_{N_j}\leq p^0, \gamma^T1_{N_0}\leq p^j\}, \] where $1_{N_0}$ denotes the $N_0\times 1$ vector whose entries are $1$ (resp. $1_{N_j}$), $p^0=[p_1^0,\ldots,p_{N_0}^0]$ is the vector of weights of $\mu^0$ (resp. $p^j$), $\gamma 1_{N_j}\leq p^0$ means that component-wise holds the `$\leq$' (resp. $\gamma^T1_{N_0}\leq p^j$, where $\gamma^T$ is the transpose of $\gamma$),
and  $|p^0|=\sum_{n=1}^{N_0}|p^0_n|$  is the total mass of $\mu^0$ (resp. $|p^j|,|\gamma|$). The marginals are $\gamma_0:=\gamma 1_{N_j}$, and $\gamma_1:=\gamma^T1_{N_0}$.

Similar to OT, when an optimal plan $\gamma^j$ for $OPT_{\lambda}(\mu^0,\hat\mu^j)$ is not induced by a map, we can replace the target measure $\mu^j$ by an \textbf{OPT barycentric projection} $\hat{\mu}^j$ for which a  map exists. Therefore, allowing us to apply the LOPT embedding (see \eqref{eq: lopt embedding general} and \eqref{eq: lopt embedding discrete} below).
\begin{definition}\label{def: opt barycentric}
Let  $\mu^0$ and $\mu^j$ be positive discrete measures, and $\gamma^j\in\Gamma^*_\leq(\mu^0,\mu^j)$. The \textbf{OPT barycentric projection}\footnote{Notice that in \eqref{eq: discrete barycenter} we had $p_n^0=\sum_{m=1}^{N_j}\gamma^j_{n,m}$. This leads to introducing $\hat p_n^j$ as in \eqref{eq: p_n hat}. That is, $\hat p_n^j$ plays the role of $p_n^0$ in the OPT framework. However, here $\hat p_n^j$ depends on $\gamma^j$ (on its first marginal $\gamma_0^j$) and not only on $\mu^0$, and so we add a superscript `$j$'.} of $\mu^j$ \textbf{with respect to} $\mu^0$ is defined as 
\begin{align}
    &\hat{\mu}^j:=\sum_{n=1}^{N_0}\hat p_n^j\delta_{\hat{x}^j_n}, \qquad \text{ where}\label{eq: opt barycenter} \\
    &\hat p_n^{j}:=\sum_{m=1}^{N_j}\gamma^j_{n,m}, \qquad  1\leq n\leq N_0, \label{eq: p_n hat}\\
    &\hat{x}_n^j:=\begin{cases}
        \frac{1}{\hat p_n^j}\sum_{m=1}^{N_j}\gamma^j_{n,m}x_m^j &\text{if } \, {\hat p_n^j}>0\\
        x_n^0 & \text{if } \, {\hat p_n^j}=0. \label{eq: opt barycentric x}
    \end{cases} 
\end{align}
\end{definition}

\begin{theorem}\label{thm: opt barycentric projection}
In the same setting of Definition \ref{def: opt barycentric}, the map $x_n^0\mapsto \hat x_n^j$ given by \eqref{eq: opt barycentric x} solves the problem $OPT_\lambda(\mu^0,\hat\mu^j),$
in the sense that induces the partial optimal plan $\hat\gamma^j=\diag(\hat p_1^j,\ldots,\hat p_{N_0}^j)$. 
\end{theorem}
It is worth noting that when we take a barycentric projection of a measure, some information is lost. Specifically, the information about the part of $\mu^j$ that is not transported from the reference $\mu^0$. This has some minor consequences.

First, unlike \eqref{eq: OT and embedding}, the optimal partial transport cost  $OPT_{\lambda}(\mu^0,\mu^j)$  changes when we replace $\mu^j$ by $\hat \mu^j$. Nevertheless, the following relation holds.
\begin{theorem}\label{thm: barycentric projection recover opt}
In the same setting of Definition \ref{def: opt barycentric}, if $\gamma^j$ is induced by a map, then 
    \begin{align}
        &OPT_{\lambda}(\mu^0,\mu^j)=OPT_\lambda(\mu^0,\hat\mu^j)+\lambda(|\mu^j|-|\hat \mu^j|) \label{eq: barycentric projection and opt}
    \end{align}
\end{theorem}
The second consequence\footnote{This is indeed an advantage since it allows the range of the embedding to always have the same dimension $N_0\times (d+1)$.} is that the LOPT embedding of $\hat{\mu}^j$ will always have a null third component. That is,
\begin{equation}
    \hat \mu^j\mapsto
    ([\hat{x}_1^j-x_1^0,\ldots,\hat{x}_{N_0}^j-x_{N_0}^0], \, \sum_{n=1}^{N_0}\hat{p}_n^j\delta_{x_n^0}, \, 0). \label{eq: lopt embedding discrete}
\end{equation}
Therefore, we represent  this embedding  as $\hat \mu^j\mapsto(u^j,\hat{p}^j)$, for $u^j=[\hat{x}_1^j-x_1^0,\ldots,\hat{x}_{N_0}^j-x_{N_0}^0]$ and  $\hat p^j=[\hat p^j_1,\ldots,\hat p^j_{N_0}]$.
The last consequence is given in the next result.
\begin{proposition} \label{pro: lopt dist barycentric}
    If $\mu^0,\mu^i,\mu^j$ are discrete and satisfy the conditions of Definition \ref{def: discrete opt embedding}, then  
    \begin{equation}
        LOPT_{\mu^0,\lambda}(\mu^i,\mu^j)=LOPT_{\mu^0,\lambda}(\hat\mu^i,\hat\mu^j)+\lambda C_{i,j}  \label{eq: lopt discrete general}
    \end{equation}
    where $C_{i,j}=|\mu^i|-|\hat \mu^i|+|\mu^j|-|\hat \mu^j|$.
\end{proposition}
As a byproduct we can define the \textbf{LOPT discrepancy} for \textbf{any} pair of discrete measures $\mu^i,\mu^j$ as the right-hand side of \eqref{eq: lopt discrete general}. In practice, unless to approximate $OPT_{\lambda}(\mu^i,\mu^j)$, we set $C_{i,j}=0$ in \eqref{eq: lopt discrete general}. That is, 
\begin{equation}\label{eq: LOPT discrete no Cij}
    LOPT_{\mu^0,\lambda}(\mu^i,\mu^j):=LOPT_{\mu^0,\lambda}(\hat\mu^i,\hat\mu^j).
\end{equation}

\subsection{OPT and LOPT Interpolation}\label{sec: interpolation}

Inspired by OT and LOT geodesics as defined in section \ref{sec: OT LOT geogesic}, but lacking the Riemannian structure provided by the OT squared norm,  we propose an OPT interpolation curve and its LOPT  approximation.

For the OPT interpolation between two measures $\mu^i$, $\mu^j$ for which exists $\gamma\in\Gamma_\leq^*(\mu^i,\mu^j)$ of the form
$\gamma=(\mathrm{id}\times T)_\#\gamma_0 \, $, a natural candidate is the solution $\rho_t$ of the dynamic formulation of $OPT_\lambda(\mu^i,\mu^j)$. The exact expression is given by Proposition \ref{pro: opt dynamic}. When working with general discrete measures $\mu^i$, $\mu^j$ (as in \eqref{eq: discrete measures}, with `$i$' in place of $0$) such $\gamma$ is not guaranteed to exist. Then,  we replace the latter with its OPT barycentric projection with respect to $\mu^i$. And by Theorem \ref{thm: opt barycentric projection} the map $T:x_n^i\mapsto \hat x_n^j$ solves $OPT_\lambda(\mu^i,\hat{\mu}^j)$  and the \textbf{OPT interpolating curve} is\footnote{$\hat p_n^j$ are the coefficients of $\hat\mu^j$ with respect to $\mu^i$ analogous to \eqref{eq: opt barycenter}.}
\begin{equation*}        
    t\mapsto\sum_{n=1}^{N_i}\hat{p}_n^j\delta_{(1-t)x_n^i+t T(x_n^i)}
    +(1-t)\sum_{n=1}^{N_i}(p_n^i - \hat{p}_n^j)\delta_{x_n^i}. 
\end{equation*}
When working with a multitude of measures, it is convenient to consider a reference $\mu^0$ and embed the measures in $\mathbb{R}^{(d+1)\times N_0}$ using LOPT. Hence, doing computations in a simpler space. Below we provide the LOPT interpolation.
\begin{definition}
    Given discrete measures $\mu^0,\mu^i,\mu^j$, with $\mu^0$ as the reference, let 
    $(u^i,\hat p^i),(u^j,\hat p^i)$ be the LOPT embeddings of $\mu^i,\mu^j$. Let $\hat{p}^{ij}:=\hat{p}^i\wedge \hat{p}^j$, and $u_t:=(1-t)u^i+tu^j \label{eq: lopt u_t}$. We define the \textbf{LOPT interpolating curve} between $\mu^i$ and $\mu^j$ by     
    \begin{align}
        t\mapsto&\sum_{k\in D_{T}}\hat p^{ij}_k\delta_{x_k^0+u_t(k)}+(1-t)\sum_{k\in D_D}(\hat p_k^i-\hat p^{ij}_k)\delta_{x_k^0+u_k^i} +t\sum_{k\in D_C}(\hat p_k^j-\hat p_k^{ij})\delta_{x^0_k+u_k^j
        }\nonumber 
    \end{align}
    where $D_T = \{k: \hat{p}_k^{ij}>0\}$, $D_D = \{k: \hat{p}_k^i>\hat{p}_k^{ij})\}$ and $D_C = \{k: \hat{p}_k^{ij}<\hat{p}_k^j)\}$ are respectively the sets where we transport, destroy and create mass. 
\end{definition}

\section{Applications}\label{sec: applications}

\textbf{Approximation of OPT Distance:}
Similar to LOT \cite{wang2013linear}, and Linear Hellinger Kantorovich (LHK) \cite{cai2022linearized}, we test the approximation performance of OPT using LOPT. Given $K$ empirical measures $\{\mu^i\}_{i=1}^K$, for each pair $(\mu^i,\mu^j)$, we compute  $OPT_{\lambda}(\mu^i,\mu^j)$ and
$LOPT_{\mu^0,\lambda}(\mu^i,\mu^j)$ and the mean or median of all pairs $(\mu^i,\mu^j)$ of relative error defined as 
$$ \frac{|OPT_{\lambda}(\mu^i,\mu^j)-LOPT_{\mu^0,\lambda}(\mu^i,\mu^j)|}{OPT_\lambda(\mu^i,\mu^j)}.$$
Similar to LOT and LHK, the choice of $\mu^0$ is critical for the accurate approximation of OPT. If $\mu^0$ is far away from $\{\mu^i\}_{i=1}^K$, the linearization is a poor approximation because the mass in $\mu^i$ and $\mu^0$ would only be destroyed or created. In practice, one candidate for $\mu^0$ is the barycenter of the set of measures $\{\mu^i\}$. The OPT can be converted into OT problem \cite{caffarelli2010free},
~and one can use OT barycenter \cite{cuturi2014fast} to find $\mu^0$. 

For our experiments, we created $K$ point sets of size $N=500$
for $K$ different Gaussian distributions in $\mathbb{R}^2$. In particular, $\mu^i\sim \mathcal{N}(m^i,I)$, where $m^i$ is randomly selected such that $\|m ^i\| = \sqrt{3}$ for $i=1,...,K$. For the reference, we picked an $N$ point representation of $\mu^0\sim \mathcal{N}(\overline{m},I)$ with $\overline{m}=\sum m^i/K$. We repeated each experiment $10$ times.
To exhibit the effect of the parameter $\lambda$ in the approximation, the relative errors are shown in Figure \ref{fig: error_vs_lambda}. For the histogram of the relative errors for each value of $\lambda$ and each number of measures $K$, we refer to Figure \ref{fig: hist} in the Appendix \ref{sec: appendix app}.
For large $\lambda$, most mass is transported and $OT(\mu^i,\mu^j)\approx OPT_{\lambda}(\mu^i,\mu^j)$, the performance of LOPT is close to that of LOT, and the relative error is small. 

\begin{figure}
\centering
\begin{subfigure}[b]{0.49\textwidth}
        \centering
    \includegraphics[width=1\textwidth]{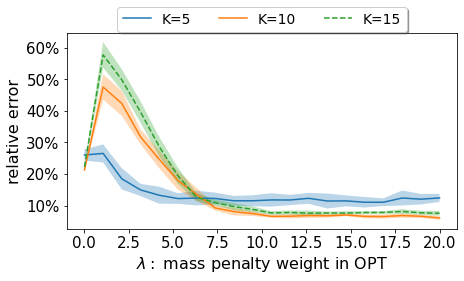} 
        \caption{Mean of the error}
        \label{fig: mean_error}
    \end{subfigure}
    \begin{subfigure}[b]{0.49\textwidth}
        \centering
        \includegraphics[width=1\textwidth]{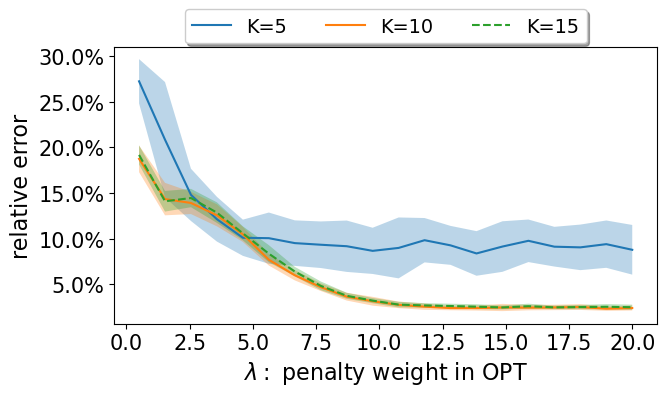}
        \caption{Medean of the error}
        \label{fig: medean_error}
    \end{subfigure}
    
    \caption{Graphs of the mean and median relative errors between $OPT_\lambda$ and $LOPT_{\lambda,\mu_0}$ as a function of the parameter $\lambda$.}
    \label{fig: error_vs_lambda}
\end{figure}

In Figure \ref{fig: wall clock time} we report wall clock times of OPT vs LOPT for $\lambda =5$. We use linear programming \cite{karmarkar1984new} to solve each OPT problem with a cost of $\mathcal{O}(N^3\text{log}(N))$ each. Thus, computing the OPT distance pair-wisely for $\{\mu^i\}_{i=1}^K$ requires $\mathcal{O}(K^2 N^3\text{log}(N))$. In contrast, to compute $LOPT$, we only need to solve $K$ optimal partial transport problems for the embeddings (see \eqref{eq: lopt embedding general} or \eqref{eq: lopt embedding discrete}). Computing LOPT discrepancies after the embeddings is linear. Thus, the total computational cost is $\mathcal{O}(K N^3\text{log}(N)+K^2N)$. The experiment was conducted on a Linux computer with AMD EPYC 7702P CPU with 64 cores and 256GB DDR4 RAM.
\begin{figure}[t!]
    \centering
    \includegraphics[width=0.7\textwidth]{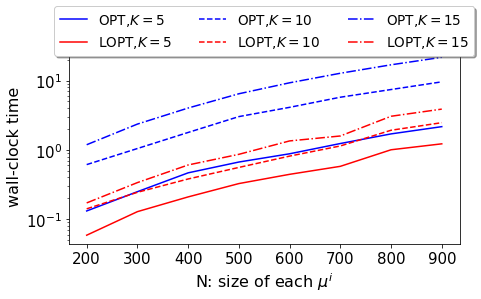}
    \caption{Wall-clock time between OPT and LOPT.  
    The LP solver in PythonOT \cite{flamary2021pot} is applied to each individual OPT problem, with $100N$ maximum number of iterations.}
    \label{fig: wall clock time}
\end{figure}

\textbf{Point Cloud Interpolation: }
We test OT geodesic, LOT geodesic, OPT interpolation, and LOPT interpolation on the  \hyperlink{https://www.kaggle.com/datasets/cristiangarcia/pointcloudmnist2d}{point cloud MNIST} dataset. We compute different transport curves between point sets of the digits $0$ and $9$. Each digit is a weighted point set $\{x_n^j,p_n^j\}_{n=1}^{N_j}$, $j=1,2$, that we consider as a discrete measure of the form $\mu^j =\sum_{n=1}^{N_j}p_n^j\delta_{x^j_n}+1/N_j\sum_{m=1}^{\eta N_j}\delta_{y_m^j}$, where the first sum corresponds to the clean data normalized to have total mass 1, and the second sum is constructed with samples from a uniform distribution acting as noise with total mass $\eta$. For HK, OPT, LHK and LOPT, we use the distributions $\mu^j$ without re-normalization, while for OT and LOT, we re-normalize them. The reference in LOT, LHK and LOPT is taken as the OT barycenter of a sample of the digits 0, 1, and 9 not including the ones used for interpolation, and normalized to have unit total mass. We test for $\eta=0,0.5,0.75$ (see Figure \ref{fig: interpolation_09_01} in the Appendix \ref{sec: appendix app}). The results for $\eta=0.5$ are shown in  Figure \ref{fig: geodesic1}. We can see that OT and LOT do not eliminate noise points. HK, OPT still retains much of the noise because interpolation is essentially between $\mu^1$ and $\hat\mu^2$ (with respect to $\mu^1$). So $\mu^1$ acts as a reference that still has a lot of noise. In LHK, LOPT, by selecting the same reference as LOT we see that the noise significantly decreases. 
In the HK and LHK cases, we notice not only how the masses vary, but also how their relative positions change
obtaining a very different configuration at the end of the interpolation. OPT and LOPT instead returns a more natural interpolation
because of the mass preservation of the transported portion and the decoupling between transport and destruction/creation of
mass.



\begin{figure}[t!]
    \centering
    \hspace{-1.5em}\includegraphics[width=\textwidth]{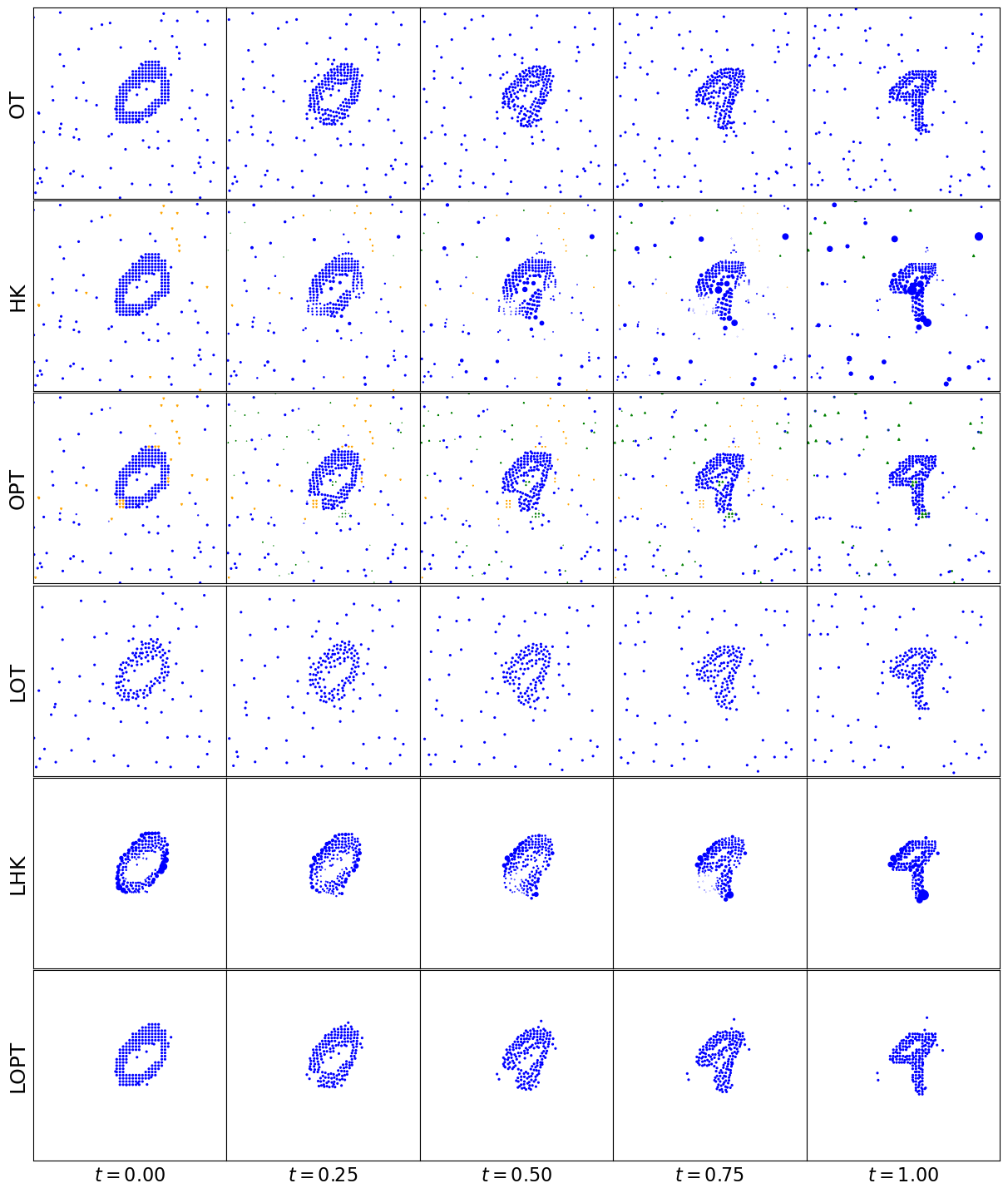}
    \caption{We demonstrate the OT geodesic, OPT interpolation, LOT geodesic and LOPT interpolation in \href{https://www.kaggle.com/datasets/cristiangarcia/pointcloudmnist2d}{MNIST} dataset. In LOT geodesic and LOPT interpolation, we use the same reference measure. The percentage of noise $\eta$ is set to $0.5$. In OPT and LOPT interpolation, we set $\lambda=20$; in HK and LHK, we set the scaling to be $2.5$.}
    \label{fig: geodesic1}
\end{figure}

\textbf{PCA analysis:}
We compare the results of performing PCA on the embedding space of LOT, LHK and LOPT for \hyperlink{https://www.kaggle.com/datasets/cristiangarcia/pointcloudmnist2d}{point cloud MNIST}.  
We take 900 digits from the dataset corresponding to digits $0,1$ and $3$ in equal proportions. Each element is a point set $\{x^j_n\}_{n=1}^{N_j}$ that we consider as a discrete measure with added noise. The reference, $\mu^0$, is set to the OT barycenter of 30 samples from the clean data. For LOT we re-normalize each $\mu^j$ to have a total mass of 1, while we do not re-normalize for LOPT. Let $S_\eta :=\{\mu^j: \text{noise level} = \eta\}_{j=1}^{900}$.  We embed $S_\eta$ using LOT, LHK and LOPT and apply PCA on the embedded vectors $\{u^j\}$. In Figure \ref{fig: pca result} we show the first two principal components of the set of embedded vectors based on LOT, LHK and LOPT for noise levels $\eta=0, 0.75$. It can be seen that when there is no noise, the PCA dimension reduction technique works well for all three embedding methods. When $\eta = 0.75$, the method fails for LOT embedding, but the dimension-reduced data is still separable for LOPT and LHK. For the running time, LOT, LOPT requires 60-80 seconds and LHK requires about 300-350 seconds. The experiments are conducted on a Linux computer with AMD EPYC 7702P CPU with 64 cores and 256GB DDR4 RAM. 
\begin{figure}[t!]
\centering
\includegraphics[width=\textwidth]{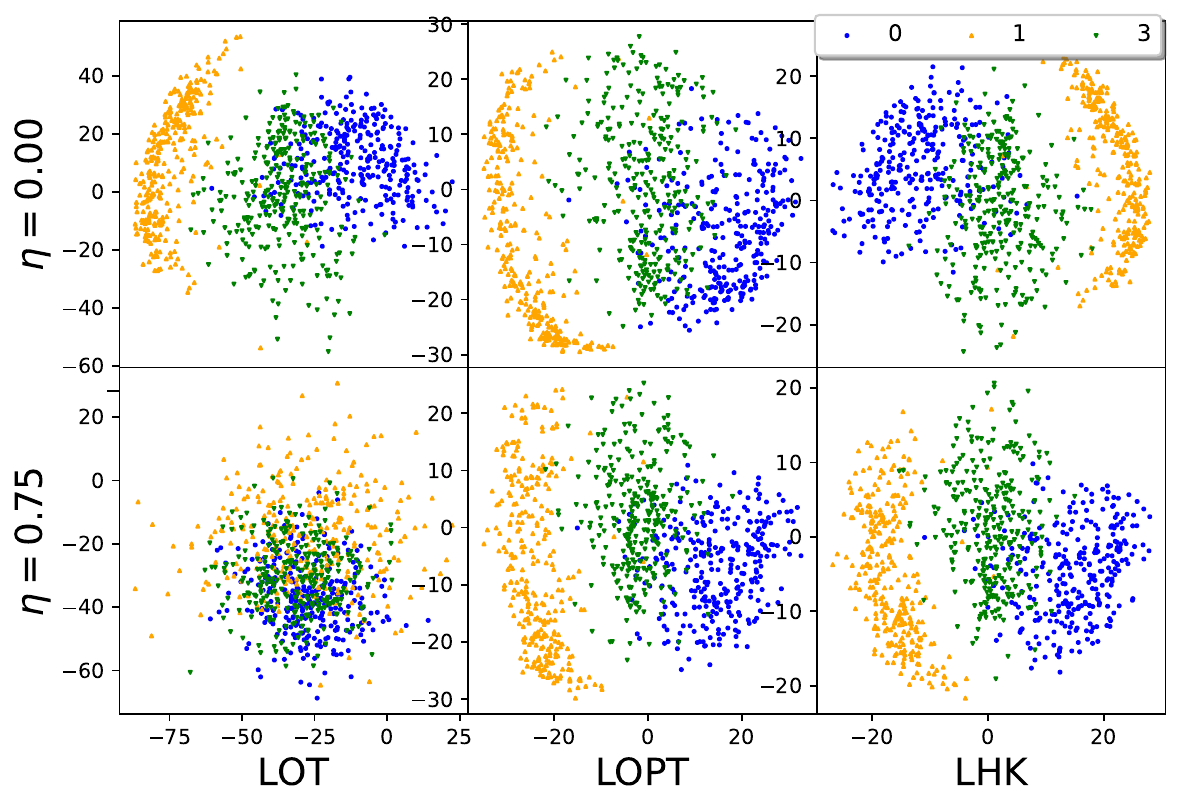}
\caption{We plot the first two principal components of each $u^j$ based on LOT and LOPT. For LOPT, we set $\lambda=20.0$, and for LHK, we set the scaling to be $2.5$.}
\label{fig: pca result}
\end{figure}

We refer the reader to Appendix \ref{sec: appendix app} for further details and analysis. 

\section{Summary}
We proposed a Linear Optimal Partial Transport (LOPT) technique that allows us to embed distributions with different masses into a fixed dimensional space
in which several calculations are significantly simplified. We show how to implement this for real data distributions allowing us to reduce the computational cost in applications that would benefit from the use of optimal (partial) transport. We finally provide comparisons with previous techniques and show some concrete applications. In particular, we show that LOPT is more robust and computationally efficient in the presence of noise than previous methods. For future work, we will continue to investigate the comparison of LHK and LOPT, and the potential applications of LOPT in other machine learning and data science tasks, such as Barycenter problems, graph embedding, task similarity measurement in transfer learning, and so on. 

\section*{Acknowledgements}
This work was partially supported by the Defense Advanced Research Projects Agency (DARPA) under Contracts No. HR00112190132 and No. HR00112190135. Any opinions, findings, conclusions, or recommendations expressed in this material are those of the authors and do not necessarily reflect the views of the United States Air Force, DARPA, or other funding agencies.

\bibliography{documents/references}
\bibliographystyle{documents/ieee_fullname}

\newpage
\appendix
\onecolumn



\begin{center}
\textbf{Appendix}
\end{center}
We refer to the main text for references.
\section{Notation}
\begin{itemize}
    \item $(\mathbb{R}^d,\|\cdot\|)$, $d$-dimensional Euclidean space endowed with the standard Euclidean norm $\|x\|=\sqrt{\sum_{k=1}^d|x(k)^2|}$.  $d(\cdot,\cdot)$ is the associated distance, i.e., $d(x,y)=\|x-y\|$. 
    \item $x\cdot y$ canonical inner product in $\mathbb{R}^d$.
    \item $1_{N}$, column vector of size $N\times 1$ with all entries equal to 1.
    \item $\diag(p_1,\ldots,p_N)$, diagonal matrix.
    \item $w^T$, transpose of a matrix $w$.  
    \item $\mathbb{R}_+$, non-negative real numbers.
    \item $\Omega$, convex and compact subset of $\mathbb{R}^d$. $\Omega^2=\Omega\times \Omega$.
    \item $\mathcal{P}(\Omega)$, set of probability Borel measures defined in $\Omega$.
    \item $\mathcal{M}_+(\Omega)$, set of positive finite Borel measures defined in $\Omega$.
    \item $\mathcal{M}$, set of signed measures.
    \item $\pi_0:\Omega^2\to\Omega$, $\pi_0(x^0,x^1):=x^0$;  $\pi_1:\Omega^2\to\Omega$, $\pi_1(x^0,x^1):=x^1$, standard projections. 
    \item $F_\#\mu$, push-forward of the measure $\mu$ by the function. $F_\#\mu(B)=\mu(F^{-1}(B))$ for all measurable set $B$, where $F^{-1}(B)=\{x: \, F(x)\in B\}$. 
    \item $\delta_x$, Dirac measure concentrated on $x$.
    \item $\mu=\sum_{n=1}^{N}p_n\delta_{x_n}$, discrete measure ($x_n\in\Omega$, $p_n\in\mathbb{R}_+$). The coefficients $p_n$ are called the weights of $\mu$.
    \item $\supp(\mu)$, support of the measure $\mu$.
    \item $\mu^i\wedge\mu^j$ minimum measure between $\mu^i$ and $\mu^j$; $p^i\wedge p^j$ vector having at each $n$ entry  the minimum value $p^i_n$ and $p^j_n$. 
    \item $\mu^0$ reference measure. $\mu^i,\mu^j$ target measures. In the OT framework, they are in $\mathcal{P}(\Omega)$. In the OPT framework, they are in $\mathcal{M}_+(\Omega)$. 
    
    \item $\Gamma(\mu^0,\mu^j)=\{\gamma\in\mathcal{P}(\Omega^2): \, \pi_{0\#}\gamma=\mu^0, \,  \pi_{1\#}\gamma=\mu^j\}$, set of Kantorovich transport plans.
    \item $C(\gamma;\mu^0,\mu^j)=\int_{\Omega^2} \|x^0-x^j\|^2 d\gamma(x^0,x^j)$, Kantorovich cost given by the transportation plan $\gamma$ between $\mu^0$ and $\mu^j$.
    \item $\Gamma^*(\mu^0,\mu^j)$, set of optimal Kantorovich transport plans.
    \item $T$, optimal transport Monge map.
    \item $\mathrm{id}$, identity map $\mathrm{id}(x)=x$.    
    \item $T_t(x)=(1-t)x+tT(x)$ for $T:\Omega\to\Omega$. Viewed as a function of $(t,x)$, it is a flow. 
    \item In section \ref{sec: background}: $\rho\in\mathcal{P}([0,1]\times\Omega)$ curve of measures. At each $0\leq t\leq$, $\rho_t\in\mathcal{P}(\Omega)$. In section \ref{sec: LOPT}: analogous, replacing $\mathcal{P}$ by $\mathcal{M}_+$.
    \item $v:[0,1]\times\Omega\to\mathbb{R}^d$. at each time $v_t:\Omega\to\mathbb{R}^d$ is a vector field. $v_0$, initial velocity.
    \item $\nabla\cdot V$, divergence of the vector field $V$ with respect to the spatial variable $x$. 
    \item $\partial_t\rho+\nabla\cdot \rho v=0$, continuity equation.
    \item $\mathcal{CE}(\mu^0,\mu^j)$, set of solutions of the continuity equation with boundary conditions $\mu^0$ and $\mu^j$.
    \item $\partial_t\rho+\nabla\cdot \rho v=\zeta$, continuity equation with forcing term $\zeta$.
    \item $\mathcal{FCE}(\mu^0,\mu^j)$, set   solutions $(\rho,v,\zeta)$ of the continuity equation with forcing term with boundary conditions $\mu^0$ and $\mu^j$. 
    \item $OT(\mu^0,\mu^j)$, optimal transport minimization problem. See \eqref{eq: OT} for the Kantorovich formulation, and \eqref{eq: OT dynamic} for the dynamic formulation.
    \item $W_2(\mu^0,\mu^j)=\sqrt{OT(\mu^0,\mu^j)}$, $2$-Wasserstein distance.
    \item $\T_{\mu}=L^2(\Omega;\mathbb{R}^d,\mu)$, where $
  \|u\|_{\mu}^2=\int_\Omega \|u(x)\|^2d\mu(x)$  (if $\mu$ is discrete, $\T_{\mu}$ is identified with $\mathbb{R}^{d\times N}$ and $\|u\|_{\mu}^2=\sum_{n=1}^N\|u(n)\|^2p_n$).
    \item $LOT_{\mu^0}(\mu^i,\mu^j)$, see \eqref{eq: LOT continuous} for continuous densities, and \eqref{eq: discrete LOT} for discrete measures. 
    \item $\lambda>0$, penalization in OPT. 
    \item $\mu\leq\nu$ if $\mu(B)\leq \nu(E)$ for all measurable set $E$ and we say that $\mu$ is dominated by $\nu$.
    \item $\Gamma_{\leq}(\mu^0,\mu^j)=\{\gamma\in\mathcal{M}_+(\Omega^2): \,  \pi_{0\#}\gamma\leq\mu^0, \,  \pi_{1\#}\gamma\leq\mu^j\}$, set of partial transport plans.
    \item $\gamma_0:=\pi_{0\#}\gamma$, $\gamma_1:=\pi_{1\#}\gamma$, marginals of $\gamma\in\mathcal{M}_+(\Omega^2)$.
    \item $\nu_0=\mu^0-\gamma_0$ (for the reference $\mu^0$), $\nu^j=\mu^j-\gamma_1$ (for the target $\mu^j$). 
    \item $C(\gamma;\mu^0,\mu^j,\lambda)=\int \|x^0-x^j\|^2 d\gamma(x^0,x^j)  +\lambda(|\mu^0-\gamma_0|+|\mu^j-\gamma_1|) $, partial transport cost given by the partial transportation plan $\gamma$ between $\mu^0$ and $\mu^j$ with penalization $\lambda$.
    \item $\Gamma_{\leq}^*(\mu^0,\mu^j)$, set of optimal partial transport plans.

    \item $OPT(\mu^0,\mu^j)$, partial optimal transport minimization problem between $\mu^0$ and $\mu^j$. See \eqref{eq: OPT} for the static formulation, and \eqref{eq: opt dynamic} for the dynamic formulation.
    
    \item $\|v\|_{\mu,2\lambda}^2:=\int_\Omega \min(\|v\|^2,2\lambda)d\mu$
    
    \item $LOPT_\lambda(\mu^i,\mu^j)$, see \eqref{eq: LOPT general}, and \eqref{eq: LOPT discrete no Cij} and the discussion above.

    \item $\mu^j\mapsto u^j$, LOT embedding (fixing first a reference $\mu^0\in\mathcal{P}(\Omega)$). If $\mu^0$ has continuous density, $u^j:=v_0^j=T^j-\mathrm{id}$, and so it is a map from measures  $\mu^j\in \mathcal{P}(\Omega)$ to vector fields $u^i$ defined in $\Omega$, see \eqref{eq: linear OT embedding continuous}.  For discrete measures ($\mu^0$, $\mu^j$ discrete), LOT embedding is needed first, and in this case, the embedding is a map from discrete probability to $\mathbb{R}^{d\times N_0}$, see \eqref{eq: LOT embedding}.

     \item $\mu$, OT and OPT barycentric projection of $\mu$ with respect to a reference $\mu^0$ (in $\mathcal{P}$ or $\mathcal{M}_+$, resp.), see \eqref{eq: discrete barycenter} and  \eqref{eq: opt barycenter}, respectively. $\hat x_n$ denote the new locations where $\hat \mu$ is concentrated. $p_n^0$ and $\hat p_n$ denote the wights of $\hat \mu$ for OT and OPT, respectively.

    \item $u_t=(1-t)u^i+tu^j$, geodesic in LOT embedding space.

    \item $\hat \rho_t$, LOT geodesic.

    \item LOPT embeding, see \eqref{eq: lopt embedding general}, and \eqref{eq: lopt embedding discrete}.
    
    \item OPT interpolating curve, and LOPT interpolating curve are defined in section \ref{sec: interpolation}.
   
\end{itemize}

\section{Proof of lemma \ref{lem: ot barycentric projection}}

\begin{proof}

We fix $\gamma^*\in \Gamma^*(\mu^0,\mu^j)$\footnote{Although in subsection \ref{subsec: LOT Discrete} we use the notation $\gamma^j\in\Gamma^*(\mu^0,\mu^j)$, here we use the notation $\gamma^*$ to emphasize that we are fixing an \textbf{optimal} plan.}, and then we compute the map $x_n^0\mapsto \hat x_n^j$ according to \eqref{eq: discrete barycenter}. This map induces a transportation plan $\hat{\gamma}^*:=\diag({p}_1^0,\ldots,p_{N_0}^0)\in \mathbb{R}_+^{N_0\times N_0}$. We will prove that $\hat \gamma^*\in\Gamma^*(\mu^0,\hat \mu^j)$. By definition, it is easy to see that its marginals are $\mu^0$ and $\hat\mu^j$. The complicated part is to prove optimality. Let $\hat \gamma\in\Gamma(\mu^0,\hat \mu^j)$ be an arbitrary transportation plan between $\mu^0$ and $\mu^j$ (i.e.
$\gamma 1_{N_0}=\gamma^T 1_{N_0}=p^0$). We need to show  
\begin{align}
    C(\hat\gamma^*;\mu^0,\hat\mu^j)\leq C(\hat\gamma;\mu^0,\hat\mu^j) \label{pf: gamma_hat is optimal}. 
\end{align}
Since $\hat\gamma^*$ is a diagonal matrix with positive diagonal, its inverse matrix is $(\hat\gamma^*)^{-1}=\diag(1/p_1^0,\ldots ,1/p_{N_0}^0)$.
We define 
$$\gamma:=\hat\gamma(\hat{\gamma}^*)^{-1}\gamma^*\in \mathbb{R}_+^{N_0\times N_j}.$$ 
We claim $\gamma\in \Gamma(\mu^0,\mu^j)$. Indeed, 
\begin{align}
\gamma 1_{N_j}&=\hat{\gamma} (\hat{\gamma}^*)^{-1}\gamma^*1_{N_j}=\hat\gamma(\hat{\gamma}^*)^{-1}p^0=\hat\gamma 1_{N_0}=p^0 \label{pf: gamma margin 1}\\ 
\gamma^T 1_{N_0}&=(\gamma^*)^T (\hat{\gamma}^*)^{-1}\hat\gamma^T 1_{N_0}=(\gamma^*)^T (\hat{\gamma}^*)^{-1}p^0=(\gamma^*)^T 1_{N_0}=p^j, \label{pf: gamma margin 2}
\end{align}
where the second equality in \eqref{pf: gamma margin 1} and the third equality in \eqref{pf: gamma margin 2} follow from the fact $\gamma^*\in \Gamma(\mu^0,\mu^1)$, and the fourth equality in \eqref{pf: gamma margin 1} and the second equality in \eqref{pf: gamma margin 2} hold since $\hat\gamma\in \Gamma(\mu^0,\hat\mu^j)$. 

Since $\gamma^*$ is optimal for $OT(\mu,\mu^j)$,we have $C(\gamma^*;\mu^0,\mu^j)\leq C(\gamma;\mu^0,\mu^j)$ to denote the transportation cost for $\gamma^*$ and $\gamma$, respectively. We write 
\begin{align}    C(\gamma^*;\mu^0,\mu^j)&=\sum_{n=1}^{N_0}\sum_{m=1}^{N_j}\|x_n^0-x_m^j\|^2\gamma^*_{n,m}\nonumber \\
&=\sum_{n=1}^{N_0}\|x_n^0\|^2p_n^0+\sum_{m=1}^{N_j}\|x_m^j\|^2p_m^j-2\sum_{n=1}^{N_0}\sum_{m=1}^{N_j}x_n\cdot x_m^j \, \gamma_{n,m}^* \nonumber \\
    &=K_1-2\sum_{n=1}^{N_0}\sum_{m=1}^{N_j}\hat{x}_n\cdot x_m^j \, \gamma_{n,m}^*, \nonumber 
\end{align}
where $K_1$ is a constant (which only depends on $\mu^0,\mu^j$). Similarly,
\begin{align}
    C(\gamma;\mu^0,\mu^j)&=K_1-2\sum_{n=1}^{N_0}\sum_{m=1}^{N_j}x_n^0\cdot x_m^j \, \gamma_{n,m} \nonumber \\
    &=K_1-2\sum_{n=1}^{N_0}\sum_{m=1}^{N_j}\sum_{\ell=1}^{N_0}\frac{\hat\gamma_{n,\ell}\gamma^*_{\ell,m}}{p^0_\ell} \, x_n^0\cdot x_m^j\nonumber 
\end{align}
Analogously, we have 
\begin{align}
    C(\hat\gamma^*;\mu^0,\hat\mu^j)&=\sum_{n=1}^{N_0}\sum_{m=1}^{N_0} \|\hat{x}_n-\hat x_{m}^j\|^2\hat\gamma_{n,m}^* \nonumber \\
&=\sum_{n=1}^{N_0}\|x_n^0\|^2p_n^0+\sum_{m=1}^{N_0}\|\hat{x}_{m}^j\|^2p_{m}^0-2\sum_{n=1}^{N_0}\sum_{m=1}^{N_0}x_n^0\cdot \hat{x}_{m}^j \, \hat\gamma^*_{n,m}\nonumber\\
    &=K_2-2\sum_{n=1}^{N_0}\sum_{m=1}^{N_j}   x^0_n\cdot x_m^j \, \gamma_{n,m}^*\nonumber\\
    &=K_2 -K_1+C(\gamma^*;\mu^0,\mu^j)
    \label{pf: C(gamma*_hat)}
\end{align}
where $K_2$ is constant (which only depends on $\mu^0$ and $\hat\mu^j$) and similarly 
\begin{align}
    C(\hat\gamma;\mu^0,\hat\mu^j)&=K_2-2\sum_{n=1}^{N_0}\sum_{m=1}^{N_j}\sum_{\ell=1}^{N_0}\frac{\hat\gamma_{n,\ell}\gamma^*_{\ell,m}}{p^0_\ell} \, x_n^0\cdot x_m^j\nonumber\\
    &= K_2-K_1 + C(\gamma;\mu^0,\mu^1)
    \label{pf: C(gamma_hat)}
\end{align}
Therefore, by \eqref{pf: C(gamma*_hat)} and \eqref{pf: C(gamma_hat)}, we have that
$C(\gamma^*;\mu^0,\mu^j)\leq C(\gamma;\mu^0,\mu^j)$ if and only if
$C(\hat\gamma^*;\mu^0,\hat\mu^j)\leq C(\hat\gamma;\mu^0,\hat\mu^j)$. So, we conclude the proof by the fact $\gamma^*$ is optimal for $OT(\mu_0,\mu^j)$.
\end{proof}
      
    


\section{Proof of Proposition \ref{thm: lot geodesic}}

\begin{lemma}\label{lem: lot geodesic}
Given discrete measures $\mu^0=\sum_{k=1}^{N_0}p^0_k \delta_{x^0_k},\mu^1=\sum_{k=1}^{N_0}p^0_k \delta_{x^1_k},\mu^2=\sum_{k=1}^{N_0}p^0_k\delta_{x^2_k}$, suppose that the maps $x_k^0\mapsto x_k^1$ and $x^0_k\mapsto x_k^2$ solve $OT(\mu^0,\mu^1)$ and $OT(\mu^0,\mu^2)$, respectively. 
For $t\in[0,1]$, define $\mu_t:=\sum_{k=1}^{N_0}p_k^0 \delta_{(1-t)x_k^1+tx_k^2}$. Then, the mapping $T_t: x_k^0\mapsto (1-t)x_k^1+tx_k^2$ solves the problem $OT(\mu^0,\mu_t)$. 
\end{lemma}
\begin{proof}
Let $\gamma^*=\diag(p_1^0,\ldots, p_{N_0}^0)$ be the corresponding transportation plan induced by $T_t$. Consider an arbitrary $\gamma\in \mathbb{R}_+^{N_0\times N_0}$ such that $\gamma\in\Gamma(\mu^0,\mu_t)$. 
We need to show
\begin{equation*}
  C(\gamma^*;\mu^0,\mu_t)\leq C(\gamma;\mu^0,\mu_t). \label{pf: gamma* }
\end{equation*}
We have 
\begin{align}
C(\gamma^*;\mu^0,\mu_t)&=\sum_{k,k'=1}^{N_0} \|x^0_k-(1-t)x^1_{k'}-tx_{k'}^2\|^2 \, \gamma^*_{k,k'} \nonumber\\ 
&=\sum_{k=1}^{N_0}\|x_k^0\|^2p_k^0+\sum_{k=1}^{N_0}\|(1-t)x_{k'}^1+tx_{k'}^2\|^2p_k^0
-2\sum_{k,k'=1}^{N_0}x_k^0\cdot ((1-t)x_{k'}^1+tx_{k'}^2) \nonumber\\ 
&=K-2\left[(1-t)\sum_{k,k'=1}^{N_0}\hat{x}_k\cdot x_{k'}^1 \, \gamma_{k,k'}^*+t \sum_{k,k'=1}^{N_0}\hat{x}_k\cdot x_{k'}^2 \, \gamma_{k,k'}^*\right], \label{pf: C(gamma*; mu0, mu1)}
\end{align}
where in third equation, $K$ is a constant which only depends on the marginals $\mu^0,\mu_t$. 
Similarly, 
\begin{align}
    C(\gamma;\mu^0,\mu_t)&=K-2\left[(1-t)\sum_{k,k'=1}^{N_0} \, \hat{x}_k\cdot x_{k'}^1\gamma_{k,k'}+t \sum_{k,k'=1}^{N_0}\hat{x}_k\cdot x_{k'}^2 \, \gamma_{k,k'}\right]. \label{pf: C(gamma; mu0, mu1)} 
\end{align}
By the fact that $\gamma,\gamma^* \in \Gamma(\hat\mu,\mu_t)=\Gamma(\mu^0,\mu^1)=\Gamma(\mu^0,\mu^2)$, and that $\gamma^*$ is optimal for $OT(\hat\mu,\mu^1),OT(\hat\mu,\mu^2)$, we have
\begin{align}
\sum_{k,k'=1}^{N_0}\hat{x}_k\cdot x_{k'}^1 \, \gamma_{k,k'}^*\ge
\sum_{k,k'=1}^{N_0}\hat{x}_k\cdot x_{k'}^1 \, \gamma_{k,k'} 
\qquad \text{ and } \qquad 
\sum_{k,k'=1}^{N_0}\hat{x}_k\cdot x_{k'}^2 \, \gamma_{k,k'}^*\ge
\sum_{k,k'=1}^{N_0}\hat{x}_k\cdot x_{k'}^2 \, \gamma_{k,k'}. \nonumber 
\end{align}

Thus, by \eqref{pf: C(gamma*; mu0, mu1)} and \eqref{pf: C(gamma; mu0, mu1)}, we have $C(\gamma^*;\hat\mu,\mu_t)\leq C(\gamma;\hat\mu,\mu_t)$, and this completes the proof. 
\end{proof}

\begin{proof}[Proof of Proposition \ref{thm: lot geodesic}]

Consider $0\leq s\leq t\leq 1$.  By Lemma \ref{lem: ot barycentric projection}, the maps  
$T^i: x^0_k\mapsto \hat{x}^i_k$, $T^j: x^0_k\mapsto \hat{x}^j_k$ solve $OT(\mu^0,\hat\mu^i),OT(\mu^0,\hat\mu^j)$, respectively. Moreover,  by Lemma \ref{lem: lot geodesic} (under the appropriate renaming), we have the mapping  
$$x^0_k\mapsto (1-s)\hat{x}^{i}_k+s\hat x_k^j=x_k^0+(1-s)u^i_k+su^j_k, \qquad 1\leq k\leq N_0$$
solves $OT(\mu^0,\hat\rho_s)$, and  similarly 
$$x^0_k\mapsto (1-t)\hat{x}^{i}_k+t\hat x_k^j=x_k^0+(1-t)u^i_k+tu^j_k, \qquad 1\leq k\leq N_0$$
solves $OT(\mu^0,\hat\rho_t)$.

Thus, 
\begin{align}
    LOT_{\mu^0}(\hat\rho_s,\hat\rho_t)&=\sum_{k=1}^{N_0}\|((1-s)\hat{x}^i_k+s\hat x_k^j)-((1-t)\hat{x}^i_k+t\hat x_k^j)\|^2 \, p_k^0  \nonumber \\ 
    &= \sum_{k=1}^{N_0}(t-s)^2\|\hat{x}_k^i-\hat{x}_k^j\|^2 \, p_k^0\nonumber\\ 
    &=(t-s)^2LOT_{\mu^0}(\mu^i,\mu^j), \nonumber 
\end{align}
and this finishes the proof. 
\end{proof}

\section{Proof of proposition \ref{pro: opt dynamic}}

This result relies on the definition \eqref{eq: opt rho} of the curve $\rho_t$, which is inspired by the fact that solutions of non-homogeneous equations are given by solutions of the associated homogeneous equation plus a particular solution of the non-homogeneous. We choose the first term of $\rho_t$ as a solution of a (homogeneous) continuity equation ($\gamma_t$ defined in \eqref{eq: gamma_t}), and the second term ($\overline{\gamma}_t:=(1-t)\nu_0 + t\nu^j$) as an appropriate particular solution of the equation  \eqref{eq: continuity equation opt} with forcing term $\zeta$ defined in \eqref{eq: opt zeta}. Moreover, the curve $\rho_t$ will become optimal for \eqref{eq: opt dynamic} since both $\gamma_t$ and $\overline{\gamma}_t$ are `optimal' in different senses. On the one hand, $\gamma_t$ is optimal for a classical optimal transport problem\footnote{Here we will strongly use that the fixed partial optimal transport plan $\gamma^*\in\Gamma_{\leq}^*(\mu^0,\mu^j)$ in the hypothesis of the proposition has the form $\gamma^*=(I\times T)_\#\gamma_0^*$, for a map $T:\Omega\to \Omega$}. On the other hand, $\overline{\gamma}_t$ is defined as a \textit{line} interpolating the new part introduced by the framework of partial transportation, `destruction and creation of mass'.

Although this is the core idea behind the proposition, to prove it we need several lemmas to deal with the details and subtleties.

First, we mention that, the  push-forward measure $F_\#\mu$ can be defined satisfying the formula of change of variables
\begin{equation*}
    \int_A g(x)\, dF_\#\mu (x)= \int_{F^{-1}(A)} g(F(x)) \, d\mu(x)
\end{equation*}
for all measurable set $A$, and all measurable function $g$, where $F^{-1}(A)=\{x: \, F(x)\in A\}$. This is a general fact, and as an immediate consequence, we have conservation of mass $$|F_\#\mu|=|\mu|.$$ 
Then, in our case, the second term in the cost function \eqref{eq: OPT cost} we can be written as $$\lambda(|\mu^0-\gamma_0|+|\mu^j-\gamma_1|=\lambda(|\mu^0|+|\mu^1|-2|\gamma|)$$
since $\gamma_0$ and $\gamma_1$ are \textit{dominated by} $\mu^0$ and $\mu^j$, respectively, in the sense that $\gamma_0\leq\mu^0$ and $\gamma_1\leq \mu^j$.

Finally, we would like to point out that when we say that  $(\rho,v,\zeta)$ is a solution for equation \eqref{eq: continuity equation opt}, we mean that it is a \textit{weak solution} or, equivalently, it is a \textit{solution in the distributional sense} \cite[Section 4 4]{sant}. That is,
for any test function $\psi:\Omega\to\mathbb{R}$ continuous differentiable with compact support, $(\rho,v,\zeta)$ satisfies
    \begin{align*}
        \partial_t\left(\int_{\Omega}  \psi \, d\rho_t\right) -\int_{\Omega}\nabla\psi\cdot v_t \, d\rho_t= \int_{\Omega} \psi\,  d\zeta_t, 
    \end{align*}
or,  equivalently, for any test function $\phi:[0,1]\times\Omega\to\mathbb{R}$
continuous differentiable with compact support, $(\rho,v,\zeta)$ satisfies 
    \begin{align*}
        \int_{[0,1]\times\Omega}\partial_t  \phi \, d\rho +\int_{[0,1]\times\Omega}\nabla\phi\cdot v \, d\rho +\int_{[0,1]\times\Omega} \,  d\zeta =\int_\Omega \phi(1,\cdot) \, d\mu^j-\int_\Omega \phi(0,\cdot) \, d\mu^0. 
    \end{align*}

\begin{lemma}\label{lem: opt solve solves ot}
If $\gamma^*$ is optimal for $OPT_{\lambda}(\mu^0,\mu^1)$, and  has marginals $\gamma_0^*$ and $\gamma_1^*$, then 
$\gamma^*$ is optimal for $OPT_\lambda(\mu^0,\gamma_1^*)$, $OPT_\lambda(\gamma_0^*,\mu^1)$ and $OT(\gamma_0^*,\gamma_1^*)$. 
\end{lemma}
\begin{proof}
Let $C(\gamma;\mu^0,\mu^1,\lambda)$ denote the objective function of the minimization problem ${OPT}_{\lambda}(\mu^0,\mu^1)$, and similarly consider $C(\gamma;\gamma_0^*,\mu^1,\lambda),C(\gamma;\mu^0,\gamma_1^*,\lambda)$. Also, let $C(\gamma,\gamma_0^*,\gamma_1^*)$ be the objective function of ${OT}(\gamma_0^*,\gamma_1^*)$. 

First, given an arbitrary plan $\gamma\in \Gamma(\gamma_0^*,\gamma_1^*)\subset \Gamma_{\leq}(\mu^0,\mu^1)$, we have 
\begin{align}
C(\gamma;\mu^0,\mu^1,\lambda)&=\int_{\Omega^2}\|x^0-x^1\|^2d\gamma(x^0,x^1)+\lambda(|\mu^0-\gamma_0|+|\mu^1-\gamma_1|)\nonumber \\
&=C(\gamma;\gamma_0^*,\gamma_1^*)+\lambda(|\mu^0-\gamma_0|+|\mu^1-\gamma_1|).\nonumber
\end{align}
Since $C(\gamma^*;\mu^0,\mu^1,\lambda)\leq C(\gamma;\mu^0,\mu^1,\lambda)$, 
we have $C(\gamma^*;\gamma_0^*,\gamma_1^*)\leq C(\gamma;\gamma_0^*,\gamma_1^*)$. Thus,  $\gamma^*$ is optimal for ${OT}(\gamma_0^*,\gamma_1^*)$.

Similarly, for every $\gamma\in\Gamma_\leq(\gamma_0^*,\mu^1)$, we have 
$C(\gamma;\mu^0,\mu^1,\lambda)=C(\gamma;\gamma_0^*,\mu^1)+\lambda|\mu^0-\gamma^*_0|$ and thus 
$\gamma^*$ is optimal for ${OPT}_\lambda(\gamma_0^*,\mu^1)$. Analogously, 
$\gamma^*$ is optimal for ${OPT}_\lambda(\mu^0,\gamma_1^*)$. 
\end{proof}

\begin{lemma}\label{lem: disjoint nu_0,nu_1} 
    If $\gamma^*\in\Gamma_{\leq}^*(\mu^0,\mu^1)$, then $d\left({\supp}(\nu^0),{\supp}(\nu^1)\right)\geq\sqrt{2\lambda}$, where $\nu^0=\mu^0-\gamma_0$, $\nu^1=\mu^1-\gamma_1$, and $d$ is the Euclidean distance in $\mathbb{R}^d$.
\end{lemma}
\begin{proof}
    We will proceed by contradiction. Assume that $d\left(\supp(\nu^0),{\supp}(\nu^1)\right)<\sqrt{2\lambda}$. Then, there exist $\widetilde{x}^0\in {\supp}(\nu^0)$ and $\widetilde{x}^1\in {\supp}(\nu^1)$ such that $\|\widetilde{x}^0-\widetilde{x}^1\|<\sqrt{2\lambda}$. Moreover, we can choose $\varepsilon>0$ such  $\nu^0(B(\widetilde{x}^0,\varepsilon))>0$ and $\nu^1(B(\widetilde{x}^1,\varepsilon))>0$, where  $B(\widetilde{x}^i,\varepsilon)$ denotes the ball in $\mathbb{R}^d$ of radius $\varepsilon$ centered at $\widetilde{x}^i$. 
    
    We will construct  $\widetilde{\gamma}\in \Gamma_{\leq}(\mu^0,\mu^1)$ with transportation cost strictly less than $OPT_\lambda(\mu^0,\mu^1)$, leading to a contradiction.
    For $i=0,1$ we denote by $\widetilde{\nu}^i$ the probability measure given by $\frac{1}{\nu^i(B(\widetilde{x}^i,\varepsilon))}\nu^i$ restricted to $B(\widetilde{x}^i,\varepsilon)$. Let $\delta=\min\{\nu^0(B(\widetilde{x}^0,\varepsilon)),\nu^1(B(\widetilde{x}^1,\varepsilon))\}$. Now, we consider    
    \begin{equation*}
        \widetilde{\gamma}:=\gamma^*+\delta (\, \widetilde{\nu}^0\times \widetilde{\nu}^1).
    \end{equation*}
    Then, $\widetilde{\gamma}\in\Gamma_\leq(\mu^0,\mu^1)$ because
    \begin{equation*}
            \pi_{i\#} \widetilde{\gamma} = \gamma_i + \delta \widetilde{\nu}^i\leq  \gamma_i + \nu^i=\mu^i \qquad \text{ for } i=0,1. 
    \end{equation*}
    The partial transport cost of $\widetilde{\gamma}$ is
    \begin{align*}                     &C(\widetilde{\gamma},\mu^0,\mu^1,\lambda)\nonumber\\
    &=
        \int_{\Omega^2}\|x^0-x^1\|^2\, d\widetilde{\gamma} + \lambda(|\mu^0| + |\mu^1| -2|\widetilde{\gamma}|)\\
        &= \int_{\Omega^2}\|x^0-x^1\|^2\, d\gamma + \delta\int_{\Omega^2}\|x^0-x^1\|^2\, d(\widetilde{\nu}^ 0\times\widetilde{\nu}^1) + \lambda(|\mu^0| + |\mu^1| -2|\gamma|-2\delta|\widetilde{\nu}^0\times\widetilde{\nu}^1|)\\
        &=OPT_\lambda(\mu^0,\mu^1)+ \delta \int_{B(\widetilde{x}^0,\varepsilon)\times B(\widetilde{x}^1,\varepsilon)}\|x^0-x^1\|^2\, d(\widetilde{\nu}^ 0\times\widetilde{\nu}^1)-2\lambda\delta\\
        &<OPT_\lambda(\mu^0,\mu^1)+ 2\lambda\delta-2\lambda\delta= OPT_\lambda(\mu^0,\mu^1),
    \end{align*}
which is a contradiction.

\end{proof}

\begin{lemma}\label{lem: gamma_t wedge nu}
    Using the notation and hypothesis of Proposition \ref{pro: opt dynamic},
    for $t\in(0,1)$ let $D_t := \{x: v_t(x) \neq 0\}$. Then, $\gamma_t \wedge \nu_0 \equiv \gamma_t \wedge \nu^j \equiv 0$ over $D_t $. 
\end{lemma}

\begin{proof}
    We will prove this for $\gamma_t\wedge \nu_0$. The other case is analogous. Suppose by contradiction that $\gamma_t\wedge \nu_0$ is not the null measure. By Lemma \ref{lem: opt solve solves ot}, we know that $\gamma^*$ is optimal for $OT(\gamma_0^*,\gamma_1^*)$.
    Since we are also assuming that $\gamma$ is induced by a map $T$, i.e. $\gamma^*=(I\times T)_\#\gamma_0^*$, then $T$ is an optimal Monge map between $\gamma_0^*$ and $\gamma_1^*$.
    Therefore, the map $T_t:\Omega\to\Omega$ is invertible (see for example \cite[Lemma 5.29]{sant}). For ease of notation we will denote $\widetilde{\nu}^0 :=\gamma_t\wedge \nu_0$ as a measure in $D_t$.  
    
    The idea of the proof is to exhibit a plan $\widetilde{\gamma}$ with smaller $OPT_\lambda$ cost than $\gamma^*$, which will be a contradiction since $\gamma^*\in\Gamma_{\leq}^*(\mu^0\mu^j)$. Let us define 
    \begin{equation*}
        \widetilde{\gamma} := \gamma^* + (I,T\circ T_t^{-1})_\# \widetilde{\nu}^0 - (I,T)_\# (T_t^{-1})_\# \widetilde{\nu}^0. 
    \end{equation*}
    Then, $\widetilde{\gamma}\in\Gamma_{\leq}(\mu^0,\mu^j)$ since
    \begin{align*}
        &\pi_{0\#} \widetilde{\gamma} =
        \pi_{0\#} \gamma^* + \widetilde{\nu}^0 -  (T_t^{-1})_\#\widetilde{\nu}^0 \leq
        \gamma_0^* + \widetilde{\nu}^0\leq
        \mu^0,\\
        &\pi_{1\#} \widetilde{\gamma} = 
        \pi_{1\#} \gamma^* + (T\circ T_t^{-1})_\#\widetilde{\nu}^0 -  (T\circ T_t^{-1})_\#\widetilde{\nu}^0 = 
        \gamma_1^*. \qquad (\text{and we know } \gamma_1^*\leq \mu^j).
    \end{align*}
    From the last equation we can also conclude that $|\widetilde{\gamma}|=|\gamma_1^*|=|\gamma^*|$. Therefore, the partial transportation cost for $\widetilde{\gamma}$ is 
    \begin{align*}
        &\int_{\Omega^2} \|x-y\|^2d\widetilde{\gamma}(x,y)+\lambda(|\mu^0|+|\mu^j|-2|\widetilde{\gamma}|)\\
        &\qquad =\underbrace{\int_{\Omega^2} \|x-y\|d\gamma^*(x,y)+\lambda(|\mu^0|+|\mu^j|-2|\gamma|)}_{OPT_\lambda(\mu^0,\mu^j)}\nonumber\\
        &\qquad\qquad+\int_{D_t} \left(\|x-T(T_t^{-1}(x))\|^2- \|T_t^{-1}(x)-T(T_t^{-1}(x))\|^2 \right)d\widetilde{\nu}^0(x)\\
        &\qquad <OPT_\lambda(\mu^0,\mu^1)       \end{align*}
    where in the last inequality we used that if $y=T_t^{-1}(x)$ we have
    $$\|T_t(y)-T(y)\|=\|(1-t)y+tT(y)-T(y)\|=(1-t)\|y-T(y)\|<\|y-T(y)\| \qquad \text{ for } t\in(0,1).$$
\end{proof}

\begin{lemma} \label{lem: null vector field}
    Using the notation and hypothesis of Proposition \ref{pro: opt dynamic}, for $t\in(0,1)$ the vector-valued measures $v_t\cdot \nu_0$ and $v_t\cdot\nu^j$  are exactly zero.
\end{lemma}

\begin{proof}
    We prove this for $v_t\cdot\nu_0$. The other case is analogous.
    We recall that the measure $v_t\cdot\nu_0$ is determined by 
    $$\int_\Omega \Phi(x)\cdot v_t(x) \, d\nu_0(x) $$
    for all measurable functions $\Phi:\Omega\to\mathbb{R}^d$, where $\Phi(x)\cdot v_t(x)$ denotes the usual dot product in $\mathbb{R}^d$ between the vectors $\Phi(x)$ and $ v_t(x)$.
    
    From Lemma \ref{lem: gamma_t wedge nu} we know that  $\gamma_t \wedge \nu_0 \equiv 0$ over $D_t$. Therefore, $\gamma_t$ and $\nu_0$ are mutually singular in that set. This implies that we can decompose $D_t$ into two disjoint sets $D_1$, $D_2$ such that $\gamma_t(D_1) = 0 = \nu_0(D_2)$. Since $v_t$ is a function defined $\gamma_t$--almost everywhere (up to sets of null measure with respect to $\gamma_t$), we can assume without loss of generality that $v_t\equiv 0$ on $D_1$. 
    
    Let $\Phi:\Omega\to\mathbb{R}^d$ be a measurable vector-valued function over $\Omega$. Using that $v_t \equiv 0 $ in $\Omega\setminus D_t$ and in $D_1$, and that $\nu_0(D_2) = 0$, we obtain
    \begin{equation*}
        \int_\Omega \Phi(x) \cdot v_t(x) \, d\nu_0(x) =\int_{\Omega\setminus D_t} \Phi(x) \cdot v_t(x) \, d\nu_0(x) + \int_{D_1} \Phi(x) \cdot v_t(x) \, d\nu_0(x) + \int_{D_2} \Phi(x) \cdot v_t(x) \, d\nu_0(x) = 0.
    \end{equation*}
    Since $\Phi$ was arbitrary, we can conclude that $v_t \cdot \nu_0\equiv 0$.
\end{proof}

\begin{lemma}\label{lem: particular solution}
    Using the notation and hypothesis of Proposition \ref{pro: opt dynamic}, the measure $\overline{\gamma}_t = (1-t)\nu_0 + t\nu^1$ satisfies the equation 
    \begin{equation}\label{eq: gamma bar}
        \begin{cases}
        \partial_t \overline{\gamma} + \nabla \cdot \overline{\gamma} v  = \zeta\\
        \overline{\gamma}_0 = \nu_0, \quad \overline{\gamma}_1 = \nu^1.
        \end{cases}
    \end{equation}
\end{lemma}

\begin{proof}
    From Lemma \ref{lem: null vector field} we have that $v_t\cdot\overline{\gamma}_t = (1-t)v_t\cdot \nu_0 + tv_t\cdot\nu^1 = 0$, then $\nabla \cdot \overline{\gamma}v =0$. It is easy to see that $\overline{\gamma}$ satisfies the boundary conditions by replacing $t=0$ and $t=1$ in its definition. Also, $\partial_t\overline{\gamma} =  \nu^1-\nu_0$. Then, since $\zeta_t = \nu^1-\nu_0$ for every $t$, we get $\partial_t\overline{\gamma}+\nabla \cdot \overline{\gamma} v  =  \nu^1  -\nu_0 + 0 = \zeta_t $.
\end{proof}

\begin{proof}[Proof of Proposition \ref{pro: opt dynamic}]\
 


First, we address that the $v:[0,1]\times \Omega\to\mathbb{R}^d$ is well defined. Indeed, by Lemma \ref{lem: opt solve solves ot}, we have that $\gamma^*=(\mathrm{id}\times T)_\# \gamma_0^*$ is optimal for $OT(\gamma_0^*,\gamma_1^*)$. Thus, for each $(t,x)\in[0,1]\times\Omega$, by so-called \textit{cyclical monotonicity property} of the support of classical optimal transport plans, there exists at most one $x_0\in\supp(\gamma_0^*)$ such that $T_t(x_0)=x$, see \cite[Lemma 5.29]{sant}. So, $v_t(x)$ in \eqref{eq: opt v} is well defined by understanding its definition as 
\begin{equation*}
v_t(x)=  \begin{cases}T(x_0)-x_0
            &\text{if }x=T_t(x_0),         \quad \text{ for } x_0\in \supp(\gamma_0^*)\\ 
            0 &\text{elsewhere}.
        \end{cases}  
\end{equation*}

Now, we check that $(\rho,v,\zeta)$ defined in  \eqref{eq: opt rho}, \eqref{eq: opt v}, \eqref{eq: opt zeta} is a solution for \eqref{eq: continuity equation opt}.
In fact, 
$$\rho_t=\gamma_t+\overline{\gamma}_t$$
where $\gamma_t$ is given by \eqref{eq: gamma_t}, and $\overline{\gamma}_t$
is as in Lemma \ref{lem: particular solution}. Then, from section \ref{sec: ot dynamic}, $(\gamma,v)$ solves the continuity equation \eqref{eq: continuity eq} since $\gamma^*=(I\times T)_\#\gamma_0^*\in \Gamma^*(\gamma_0^*,\gamma_1^*)$ (Lemma \ref{lem: opt solve solves ot}), and from its  definition $\gamma_t=(T_t)_\#\gamma_0^*$ coincides with \eqref{eq: ot rho} in this context, as well as $v_t$ coincides with \eqref{eq: ot v} (the support of $v_t$ lies on the support of $\gamma_0^*$). On the other hand,  from Lemma \ref{lem: particular solution}, $\overline{\gamma}_t$ solves \eqref{eq: gamma bar}.
Therefore, by linearity,
\begin{align*}
    \partial_t\rho_t+\nabla\cdot \rho_t v_t=\underbrace{\partial_t\gamma_t+\nabla\cdot \gamma_t v_t}_{0}+\underbrace{\partial_t\overline{\gamma}_t+\nabla\cdot \overline{\gamma} _tv_t}_{\zeta_t}=\zeta_t
\end{align*}

Finally, by plugging in  $(\rho,v,\zeta)$ into the objective function in \eqref{eq: opt dynamic}, we have: 
 \begin{align}
    \int_{[0,1]\times\Omega}\|v\|^2 d\rho+\lambda|\zeta|&=\int_{[0,1]\times\Omega}\|v\|^2 \, d\gamma_t \, dt+\lambda|\zeta|\nonumber\\
    &=\int_\Omega \|T(x^0)-x^0\|^2 \, d\gamma_0^*(x^0)+\lambda (|\nu_0|+|\nu^1|) \nonumber\\
     &=\int_{\Omega^2} \|x^0-x^j\|^2 \, d\gamma^*(x^0,x^j)+\lambda (|\mu^0-\gamma_0^*|+|\mu^j-\gamma_1^*|) \nonumber\\
     &=OPT_\lambda(\mu^0,\mu^j)\nonumber
 \end{align}
 since $\gamma^*\in\Gamma_{\leq}^*(\mu^0,\mu^j)$. So, this
 shows that $(\rho,v,\zeta)$ is minimum for \eqref{eq: opt dynamic}.
 
 The `moreover' part holds from the above identities, using that
 \begin{align*}
     &\int_\Omega \|v_0\|^2d\gamma_0^*(x_0)+\lambda (|\nu_0|+|\nu^1|)=\int_\Omega \min(\|v_0\|^2,2\lambda)d\gamma_0^*(x_0)+\lambda (|\nu_0|+|\nu^1|) \\
     & \text{ for }  v_0(x^0)=T(x^0)-x^0,
 \end{align*}
since $\gamma^*$ is such that $ \|x^0-x^j\|^2<2\lambda$ for all $ (x^0,x^j)\in \supp(\gamma)$  \cite[Lemma 3.2]{bai2022sliced}. 



\end{proof}

\section{Proof of Theorem \ref{thm: opt barycentric projection}}

The proof will be similar to that of Lemma \ref{lem: ot barycentric projection}. 
To simplify the notation, in this proof, we use $\mu^1$ (and $\hat \mu^1$) to replace $\mu^j$ (and $\hat \mu^j$). 

\subsection{Notation setup}
We choose $\gamma^1\in\Gamma_{\leq}^*(\mu^0,\mu^1)$. We will understand  the second marginal distribution $\pi_{1\#}\gamma^1$ induced by $\gamma^1$, either as the vector $\gamma^1_1:=(\gamma^1)^T 1_{N_0}$ or as  the measure 
$\sum_{m=1}^{N_1}(\gamma^1_1)_m\delta_{x_m^1}$, and, by abuse of notation, we will write $\pi_{1\#}\gamma^1=(\gamma^1)^T1_{N_0}=\gamma_1^1$ (analogously for $\pi_{0\#}\gamma^1$).

Let $\hat\gamma^1:=\diag(\hat p_1^1,\ldots,\hat p_{N_0}^1)$ denote the transportation plan induced by mapping $x_k^0\mapsto \hat x_k^1$.

Let 
$$D:=\{k\in\{1,\dots N_0\}: \hat p_k^1>0\}.$$
We select $\hat\gamma\in\Gamma_{\leq}(\mu^0,\hat\mu^1)$.  
With a slight abuse of notation, we define 
$$\gamma:=\hat\gamma (\hat{\gamma}^1)^{-1}\gamma^j$$
where we mean that 
$(\hat{\gamma}^1)^{-1}\in \mathbb{R}^{N_0\times N_0}$ is a digonal matrix with: 
$$(\hat{\gamma}^1)_{kk}=\begin{cases}
    \frac{1}{\hat{p}_k^1}&\text{if }\hat{p}_k^1>0 \nonumber\\
    0 &\text{elsewhere} 
\end{cases},\forall k\in[1:N_0].$$

\textbf{The goal is to show  
\begin{equation}\label{eq: compare cost}
C(\hat\gamma^j;\mu^0,\hat\mu^j,\lambda)\leq C(\hat\gamma;\mu^0,\hat\mu^j,\lambda).  
\end{equation}}

\subsection{Connect $OPT_\lambda(\mu^0,\mu^1)$ and $OPT_\lambda(\mu^0,\hat{\mu}^1).$}
Note, 
First, we claim $\gamma\in\Gamma_\leq (\mu^0,\gamma_1^1)\subset \Gamma_\leq (\mu^0,\mu^1)$. Indeed, we have 
\begin{align}
\gamma 1_{N_1}&=\hat\gamma (\hat \gamma^1)^{-1}\gamma ^j1_{N_1}=\hat\gamma(\hat\gamma^1)^{-1}\hat p^j= \hat\gamma 1_D=\hat\gamma_0 \leq p^0 \label{eq: aux1} \\ 
\gamma^T 1_{N_0}&=(\gamma^1)^T(\hat\gamma^1)^{-1}\hat\gamma ^T1_{N_0}\leq (\gamma^1)^T(\hat\gamma^1)^{-1} \hat p^1=(\gamma^1)^T 1_D=\gamma^1_1 \label{eq: aux2},
\end{align}
where $1_D\in \mathbb{R}^{N_0}$ is defined with $(1_D)_{k}=1$ if $k\in D$ and $(1_D)_k=0$ elsewhere. 

In \eqref{eq: aux1}, the equality $\hat\gamma 1_D=\hat \gamma_0$ follows from the following: $(\hat \gamma^1)^T1_{N_0}\leq \hat p^1$, we have $\hat{\gamma}_{k,k'}=0,\forall k'\notin D, k\in\{1,\ldots,N_0\}$.
In \eqref{eq: aux2}, the inequality follows from the fact $\hat \gamma\in \Gamma_{\leq}(\mu^0,\hat\mu^1)$.

Note, from \eqref{eq: aux1}, we also have $\gamma^1_0=\hat\gamma^1_0$. 

We compute the transportation costs induced by $\gamma^1,\gamma,\hat\gamma^1,\hat\gamma$: 
\begin{align}
&C(\gamma^1;\mu^0,\mu^1,\lambda)\nonumber\\
&=\sum_{i=1}^{N_1}\sum_{k\in D}\|x_k^0-x_i^1\|^2\gamma_{k,i}^1+\lambda(|p^0|+|p^1|-2|\gamma^1|)\nonumber\\
&=\sum_{k\in D}\|x^0_k\|^2(\gamma_0^1)_k+\sum_{i=1}^{N_1}\|x_i^1\|^2(\gamma_1^1)_i-2\sum_{k\in D}\sum_{i=1}^{N_1}x_k^0\cdot x_i^1 \, \gamma_{k,i}^1+\lambda(|p^0|+|p^1|-2|\gamma^1|)\nonumber
\end{align}
Similarly, 
\begin{align}
&C(\hat\gamma^1;\mu^0,\hat\mu^1,\lambda)\nonumber\\
&=\sum_{k\in D}\|x^0_k\|^2(\hat\gamma_0^1)_k
+\sum_{k'\in D}\|\hat x_{k'}^1\|^2(\hat\gamma^1_1)_{k'}-2\sum_{k\in D}\sum_{k'\in D}x_k^0\cdot \hat x_{k'}^1 \, \hat\gamma_{k,k'}^1+\lambda(|p^0|+|\hat p^1|-2|\hat\gamma^1|)\nonumber \\
&=\sum_{k\in D}\|x^0_k\|^2(\hat\gamma_0^1)_k
+\sum_{k'\in D}\|\hat x_{k'}^1\|^2(\hat\gamma^1_1)_{k'}
-2\sum_{k\in D}\sum_{k'\in D}\sum_{i=1}^{N_1}x_k^0\cdot x_i^1 \, \gamma^1_{k',i}\frac{\hat\gamma^1_{k,k'}}{\hat\gamma^1_{k',k'}}+\lambda(|p^0|+|\hat{p}^1|-2|\hat\gamma^1|)\nonumber\\
&=\sum_{k\in D}\|x^0_k\|^2(\hat\gamma_0^1)_k
+\sum_{k'\in D}\|\hat x_{k'}^1\|^2(\hat\gamma^1_1)_{k'}
-2\sum_{k\in D}\sum_{i=1}^{N_1}x_k^0\cdot x_{i}^1 \, \gamma^1_{k,i}
+\lambda(|p^0|+|\hat{p}^1|-2|\hat\gamma^1|)\nonumber
\end{align}
Therefore, we obtain 
\begin{align}
C(\hat\gamma^1;\mu^0,\hat\mu^1,\lambda)=C(\gamma^1;\mu^0,\gamma_1^1,\lambda)-\sum_{i=1}^{N_1}\| x_i^1\|^2(\gamma^1_1)_i+\sum_{k'\in D}\|\hat{x}_{k'}^1\|^2(\hat\gamma^1_1)_{k'}+\lambda(|\hat{p}^1|-|p^1|) \label{pf: opt C(gamma_hat *)}
\end{align}

And also,
\begin{align}
&C(\gamma;\mu^0,\mu^1,\lambda)\nonumber\\
&=\sum_{k=1}^{N_0}\|x_k^0\|^2(\gamma_0)_k+\sum_{i=1}^{N_1}\|x^1_i\|^2(\gamma_1)_i-2\sum_{k=1}^{N_0}\sum_{i=1}^{N_1}\sum_{k'\in D}x_k^0\cdot x_i^1 \, \hat\gamma_{k,k'}\frac{\gamma^1_{k',i}}{\hat\gamma_{k',k'}^1}+\lambda(|p^0|+|p^1|-2|\gamma|), \nonumber
\end{align}
\begin{align}
&C(\hat\gamma;\mu^0,\hat\mu^1,\lambda) \nonumber\\ 
&=\sum_{k=1}^{N_0}\|x_k^0\|^2(\hat\gamma_0)_k+\sum_{k'\in D}\|\hat x_{k'}^1\|^2(\hat\gamma_1)_{k'}-2\sum_{k=1}^{N_0}\sum_{k'\in D}\sum_{i=1}^{N_0}x_k^0\cdot x_i^1 \, \hat\gamma_{k,k'}^1\frac{\gamma^1_{k',i}}{\hat \gamma^1_{k',k'}}+\lambda(|p^0|+|\hat p^1|-2|\hat\gamma|)\nonumber 
\end{align}
Thus we obtain 
\begin{align}
C(\hat\gamma;\mu^0,\hat\mu^1,\lambda)&=C(\gamma;\mu^0,\gamma_1^1,\lambda)-\sum_{i=1}^{N_1}\|x_i^1\|^2(\gamma_1)_i+\sum_{k'\in D}\|\hat{x}^1_{k'}\|^2(\hat\gamma_1)_{k'}+\lambda(|\hat p^1|-|p^1|) \label{pf: opt C(gamma_hat)}    
\end{align}

Combining with \eqref{pf: opt C(gamma_hat *)} and \eqref{pf: opt C(gamma_hat)}, we have 
\begin{align}
&C(\hat\gamma^1;\mu^0,\hat\mu^1,\lambda)-C(\hat\gamma;\mu^0,\hat\mu^1,\lambda) \nonumber \\
&=C(\gamma^1;\mu^0,\gamma^1_1,\lambda)-C(\gamma;\mu^0,\gamma^1_1,\lambda)
+\sum_{i=1}^{N_1}\|x_i^1\|^2((\gamma_1)_i-(\gamma^1_1)_i)
+\sum_{k\in D}|\hat{x}_k^1|^2
((\hat\gamma^1_1)_k-(\hat\gamma_1)_k) \nonumber \\
&=\sum_{i=1}^{N_1}\|x_i^1\|^2((\gamma_1)_i-(\gamma^1_1)_i)
-\sum_{k\in D}\|\hat{x}_k^1\|^2
((\hat\gamma_1)_k-(\hat\gamma^1_1)_k)\label{pf: C(gamma_hat)-C(gamma)}
\end{align}
where the inequality holds sine $\gamma^1$ is optimal for $OPT_{\lambda}(\mu^0,\gamma^1_1)$. 

\subsection{Verification of the inequality} 
It remains to show \eqref{pf: C(gamma_hat)-C(gamma)} is less than $0$.  By Jensen's inequality, for each $k\in D$, we obtain 
$$\|\hat x^1_k\|^2\leq \sum_{i=1}^{N_1}\frac{\gamma^1_{k,i}}{\hat{p}^1_k}\|x_i^1\|^2.$$ Combined with the fact $\hat{\gamma}_1\leq \hat\gamma^1_1$, we obtain: 
\begin{align}
\eqref{pf: C(gamma_hat)-C(gamma)}
&\leq \sum_{i=1}^{N_1}\|x_i^1\|^2((\gamma_1)_i-(\gamma^1_1)_i)
-\sum_{k\in D}\sum_{i=1}^{N_1}\frac{\gamma^1_{k,i}}{\hat{p}^1_k}\|x_i^1\|^2((\hat{\gamma}_1)_k-(\hat{\gamma}^1_1)_k)\nonumber\\
&=\sum_{i=1}^{N_1}\|x_i^1\|^2\left((\gamma_1)_i-(\gamma^1_1)_i-\sum_{k\in D}\frac{\gamma_{k,i}^1((\hat\gamma_1)_k-\hat{p}^1_k)}{\hat{p}^1_k}\right)\label{pf:diff_bound}
\end{align}

Pick $i\in [1:N_0]$, we have: 
\begin{align}
&(\gamma_1)_i-(\gamma^1_1)_i-\sum_{k\in D}\frac{\gamma_{k,i}^1((\hat\gamma_1)_k-\hat{p}^1_k)}{\hat{p}^1_k}\nonumber\\
&=(\gamma_1)_{i}-\hat{p}_i^1-\sum_{k\in D}\frac{\gamma_{k,i}^1((\hat\gamma_1)_k)}{\hat{p}^1_k}+\hat{p}_i^1\nonumber\\
&=\sum_{k'=1}^N\gamma_{k',i}-\sum_{k\in D}\frac{\gamma_{k,i}^1((\hat\gamma_1)_k)}{\hat p^1_k}\nonumber\\
&=\sum_{k'=1}^{N_0}\sum_{k\in D}\frac{\hat{\gamma}_{k',k}\gamma^1_{k,i}}{\hat{p}_{k}^1}-\sum_{k\in D}\sum_{k'=1}^{N_0}\frac{\gamma^1_{k,i}\hat{\gamma}_{k',k}}{\hat{p}_k^1}\label{pf:gamma_ki}\\
&=0\nonumber 
\end{align}
where \eqref{pf:gamma_ki} holds by the fact: for each $k'\in [1:N_0], i\in [1:N_1]$, we have  
\begin{align}
\gamma_{k',i}
&=(\hat\gamma(\hat\gamma^1)^{-1}[k',:]) \gamma^1[:,i]\nonumber\\
&=\left[\hat\gamma_{k',1} \frac{1}{\hat p^1_1},\ldots \hat\gamma_{k',N^0}\frac{1}{\hat{p}^j_{N_0}}\right]^T[\gamma^1_{1,i},\ldots \gamma^1_{N_0,i}]\nonumber\\
&=\sum_{k\in D}\frac{\hat{\gamma}_{k',k}\gamma^1_{k,i}}{\hat{p}_k^1} \nonumber 
\end{align}
Thus \eqref{pf:diff_bound} is $0$. Therefore, \eqref{pf: C(gamma_hat)-C(gamma)} is upper bounded by $0$, and we complete the proof. 

\section{Proof of Theorem \ref{thm: barycentric projection recover opt}}

\begin{proof}

Without loss of generality, we assume that $\gamma^j\in\Gamma_\leq^*(\mu^0,\mu^j)$ is induced by a 1-1 map. We can suppose this since the trick is that we will end up with an $N_0$-point representation of $\mu^j$ when we get $\hat\mu^j$. For example, if two $x_n^0$ and $x_m^0$ are mapped to the same point $T(x_n^0)=T(x_m^0)$, when performing the barycentric, we can split this into two different labels with corresponding labels. (The moral is that the labels are important, rather than where the mass is located.)
Therefore, $\gamma^j\in \mathbb{R}^{N_0\times N_j}$ is such that in each row and column there exists at most one positive entry, and all others are zero. Let  $\hat\gamma:= \diag(\hat p_1^j,\ldots, \hat p_{N_0}^j)$. Then, 
by the definition of $\hat{\mu}^j$, we have $C(\gamma^j;\mu^0,\mu^j,\lambda)=C(\hat\gamma;\mu^0,\hat\mu^j,\lambda)$. 
Thus,  
\begin{align}
OPT_{\lambda}(\mu^0,\mu^j)&=C(\gamma^j;\mu^0,\mu^j,\lambda)\nonumber\\
&=C(\gamma^j;\mu^0,\gamma^j_1,\lambda)+\lambda (|\mu^j|-|\gamma_1^j|)\nonumber\\
&=C(\hat\gamma;\mu^0,\hat\mu^j,\lambda)+\lambda(|\mu^j|-|\hat\mu^j|)\nonumber \\
&=OPT_\lambda(\mu^0,\hat\mu^j)+\lambda(|\mu^j|-|\hat\mu^j|),\nonumber
\end{align}
where the last equality holds from Theorem \ref{thm: opt barycentric projection}. This concludes the proof.

Moreover, from the above identities (for discrete measure) we can express
\begin{equation*}
    OPT_\lambda(\mu^0,\mu^j)
=\sum_{k=1}^{N_0}\hat{p}_k\|x_k^0-\hat x_k^j\|^2+\lambda|p^0-\hat{p}_k^j|+\lambda(|p^j|-|\hat p^j|).
\end{equation*}
\end{proof}

\section{Proof of Proposition \ref{pro: lopt dist barycentric}}
\begin{proof} 
The result follows from \eqref{eq: lopt embedding discrete}  
and \eqref{eq: LOPT general}. Indeed, we have
\begin{align*}
    &LOPT_{\lambda,\mu_0}(\hat\mu^j,\hat\mu^j)=\|u^i-u^j\|_{\hat{p}^i\wedge\hat{p}^j,2\lambda}+\lambda(|\hat p^i-\hat p^j|)\qquad \text{and}\\
    &LOPT_{\lambda,\mu_0}(\mu^j,\mu^j)=\|u^i-u^j\|_{\hat{p}^i\wedge\hat{p}^j,2\lambda}+\lambda(|\hat p^i-\hat p^j|) +\lambda(|\nu^i|+|\nu^j|), 
\end{align*}
since the LOPT barycentric projection of $\mu^i$ is basically $\mu^i$ without the mass that needs to be created from the reference (analogously for $\mu^j$).
\end{proof}

\section{Applications}\label{sec: appendix app}

For completeness, we will expand on the experiments and discussion presented in Section \ref{sec: applications}, as well as on Figure \ref{fig: HK and OPT} which contrasts the HK technique with OPT from the point of view of interpolation of measures.

First, we recall that similar to LOT \cite{wang2013linear} and \cite{moosmuller2020linear}, the goal of LOPT is not exclusively to approximate OPT, but to propose new transport-based metrics (or discrepancy measures) that are computationally less expensive and easier to work with than OT or OPT, specifically when many measures must be compared.  

Also, one of the main advantages of the 
\textit{linearization} step is that it allows us to embed sets of probability (resp. positive finite) measures into a linear space ($\mathcal{T}_{\mu^0}$ space). Moreover, it does it in a way that allows us to use the $\mathcal{T}_{\mu^0}$-metric in that space as a proxy (or replacement) for more complicated transport metrics while preserving the natural properties of transport theory. As a consequence, data analysis can be performed using Euclidean metric in a simple vector space. 

\subsection{Approximation of OPT Distance}

For a  better understanding of the errors plotted in Figure \ref{fig: error_vs_lambda}, the following Figure \ref{fig: hist} shows the histograms of the relative errors for different values of  $\lambda$ and each number of measures $K=5,10,15$.

\begin{figure}[t]
    \centering
    \includegraphics[width=1\textwidth]{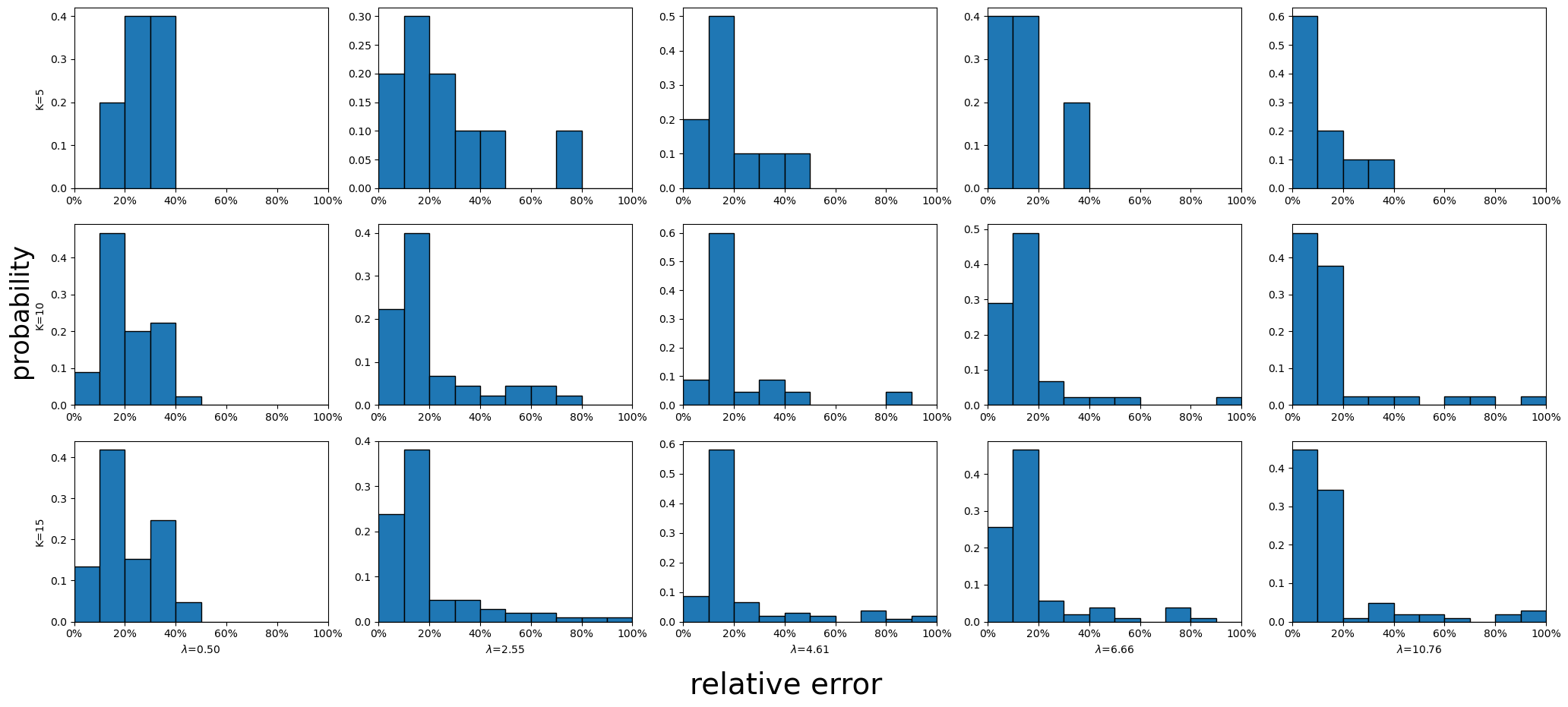}
    \caption{Histogram of relatives errors (depicted in Figure \ref{fig: error_vs_lambda}) between $OPT_\lambda$ and $LOPT_{\lambda,\mu_0}$.
    (Number of samples $N=500$. Number of repetitions $=10$.  Dimension $=2$. --measures on $\mathbb{R}^2$--.)}
    \label{fig: hist}
\end{figure}

As said in Section \ref{sec: applications}, we recall that for these experiments, we created $K$ point sets of size $N$
for $K$ different Gaussian distributions in $\mathbb{R}^2$. In particular, $\mu^i\sim \mathcal{N}(m^i,I)$, where $m^i$ is randomly selected such that $\|m ^i\| = \sqrt{3}$ for $i=1,...,K$. For the reference, we picked an $N$ point representation of $\mu^0\sim \mathcal{N}(\overline{m},I)$ with $\overline{m}=\sum m^i/K$. 
For figures \ref{fig: error_vs_lambda} and \ref{fig: hist}, the sample size $N$ was set equal to $500$. 

In what follows, we include 
tests for $N=200,250, 300, 350,...,900,950, 1000$ and $K=2,4$. For each $(N,K)$, we repeated each experiment $10$ times. The relative errors are shown in Figure \ref{fig: accuracy}. For large $\lambda$, most mass is transported and $OT(\mu^i,\mu^j)\approx OPT_{\lambda}(\mu^i,\mu^j)$, the performance of LOPT is close to that of LOT, and the relative error is small. For small $\lambda$, almost no mass is transported, $OPT_{\lambda}(\mu^i,\mu^j)\approx \lambda(|\mu^i|+|\mu^j|)\approx  \lambda(|\mu^i| - |\hat\mu^i| +|\mu^j| -|\hat\mu^j|)\approx{LOPT}_{\mu^0,\lambda}(\mu^i,\mu^j)$, and we still have a small error. In between, e.g., $\lambda=5$, we have the largest relative error. Similar results were obtained by setting the reference as the OT barycenter.

\begin{figure}
    \centering
    \includegraphics[width=0.45\textwidth]{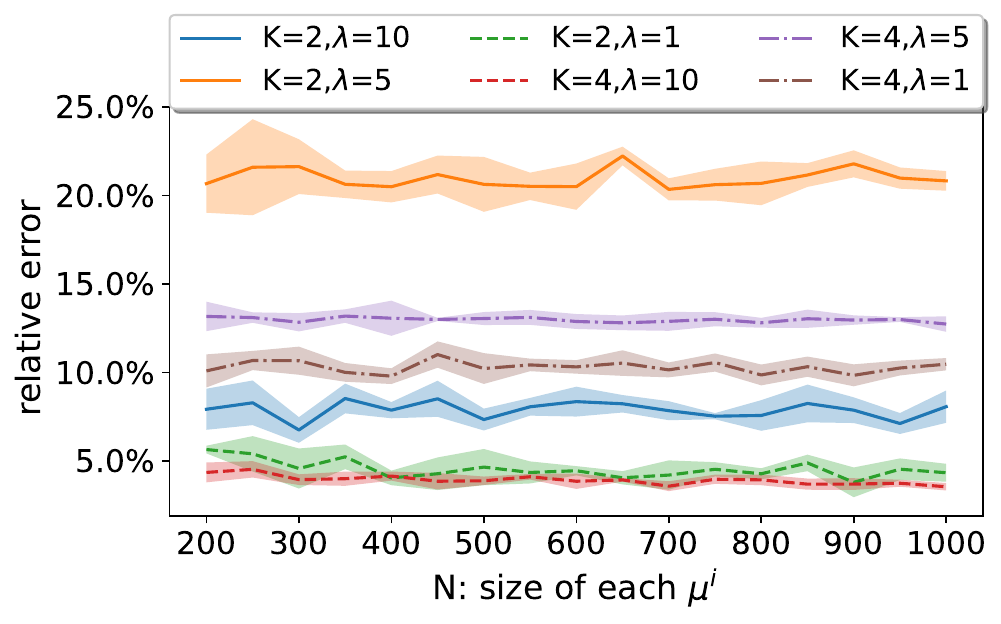}
    \caption{The average relative error between $OPT_\lambda$ and $LOPT_{\lambda,\mu_0}$ between all pairs of discrete measures $\mu^i$ and $\mu^j$.
    }
    \label{fig: accuracy}
\end{figure} 

\subsection{PCA analysis}

For problems where doing pair-wise comparisons between $K$ distributions is needed, in the classical optimal (partial) transport setting we have to solve $\binom{K}{2}$ OT (resp. OPT) problems. In the LOT (resp. LOPT) framework, however, one only needs to perform $K$ OT (resp. OPT) problems (matching each distribution with a reference measure).

One very ubiquitous example is to do clustering or classification of measures. For this case, the number of target measures $K$ representing different samples is usually very large. For the particular case of data analysis on Point Cloud MNIST, after using PCA in the embedding space (see Figure \ref{fig: pca result}), it can be observed that the LOT framework is a natural option for separating different digits, but the equal mass requirement is too restrictive in the presence of noise. LOPT performs better since it does not have this restriction. This is one example where the number of measures $K=900$
 is much larger than 2.

\subsection{Point Cloud Interpolation}

Here, we add the results of the experiments mentioned in Section \ref{sec: applications} which complete Figure \ref{fig: geodesic1}. In the new Figure \ref{fig: interpolation_09_01}, in fact, Figure \ref{fig: geodesic1} corresponds to the subfigure \ref{fig: 09_0.5}. We conducted experiments using three digits (0, 1, and 9) from the PointCloud MNIST dataset, with 300 point sets per digit. We utilized LOPT embedding to calculate and visualize the interpolation between pairs of digits. We chose the barycenter between 0, 1, and 9 as the reference for our experiment. However, to avoid redundant results, in the main paper, we only demonstrated the interpolation between 0 and 9 in our main paper. The remaining plots for the other digit pairs using different levels of noise $\eta$ are included here for completeness. Later, in Section \ref{sec: HK vs OPT}, Figure \ref{fig: all comparisons} will add the plots for the HK technique and its linearized version.

\begin{figure}
    \centering
    \begin{subfigure}[b]{0.33\textwidth}
        \centering
        \includegraphics[width=1\textwidth]{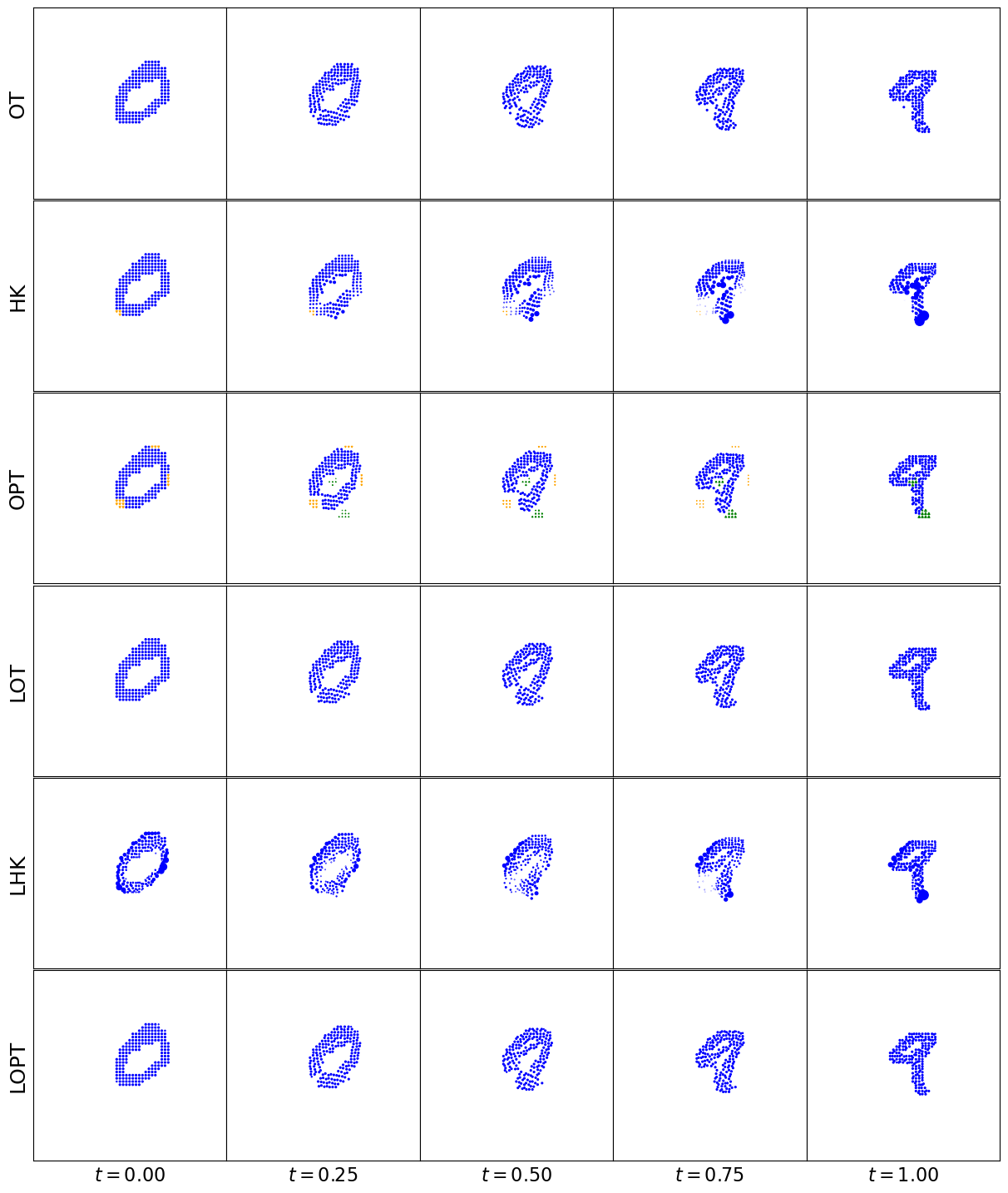} 
        \caption{$\eta=0$}
        \label{fig: 09_0}
    \end{subfigure}
    \begin{subfigure}[b]{0.33\textwidth}
        \centering
        \includegraphics[width=1\textwidth]{pic/geodesic_09_0.5.png}
        \caption{$\eta=0.5$}
        \label{fig: 09_0.5}
    \end{subfigure}
    \begin{subfigure}[b]{0.33\textwidth}
        \centering
        \includegraphics[width=1\textwidth]{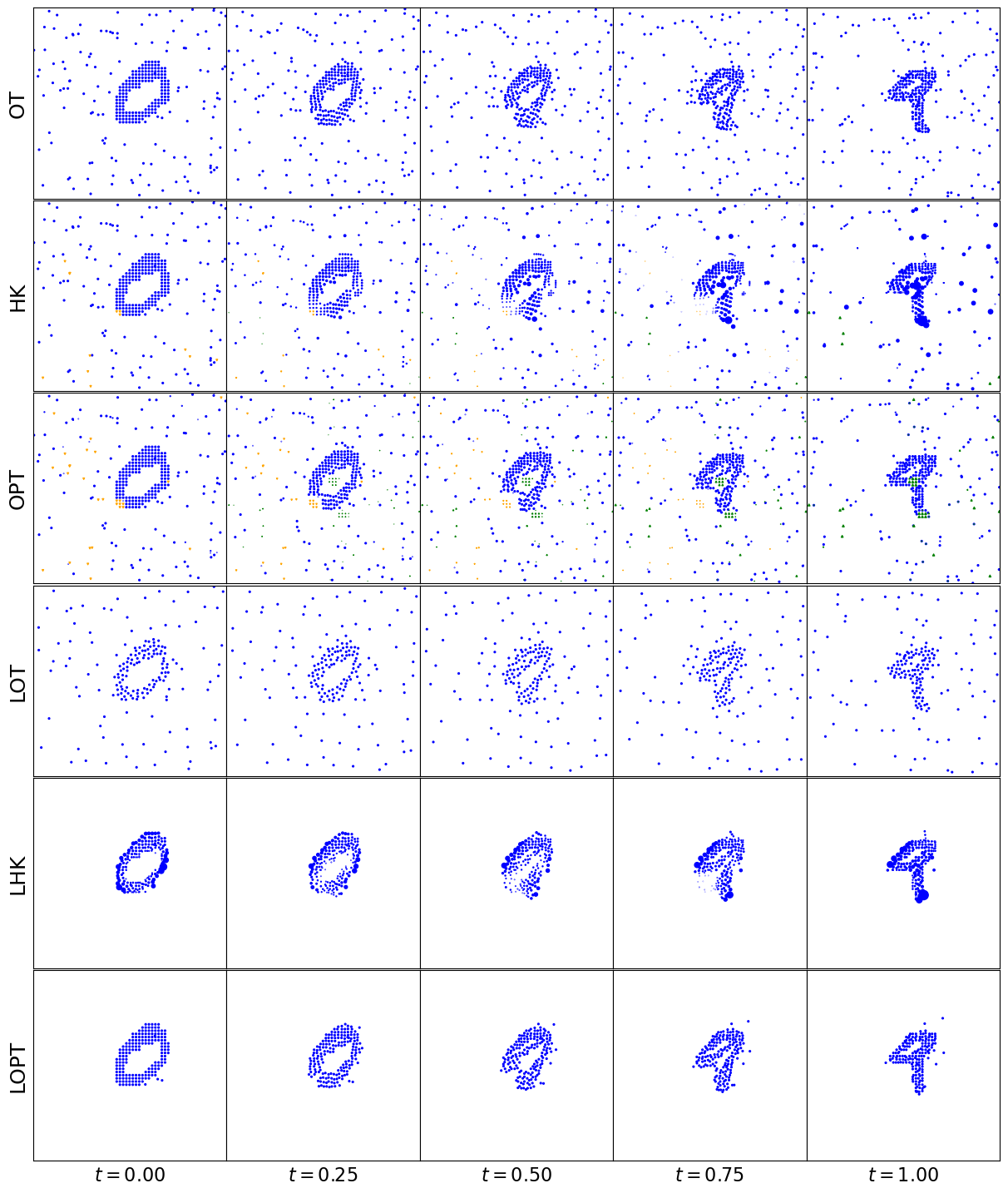}
        \caption{$\eta=0.75$}
        \label{fig: 09_0.75}
    \end{subfigure}
    \begin{subfigure}[b]{0.33\textwidth}
        \centering
        \includegraphics[width=1\textwidth]{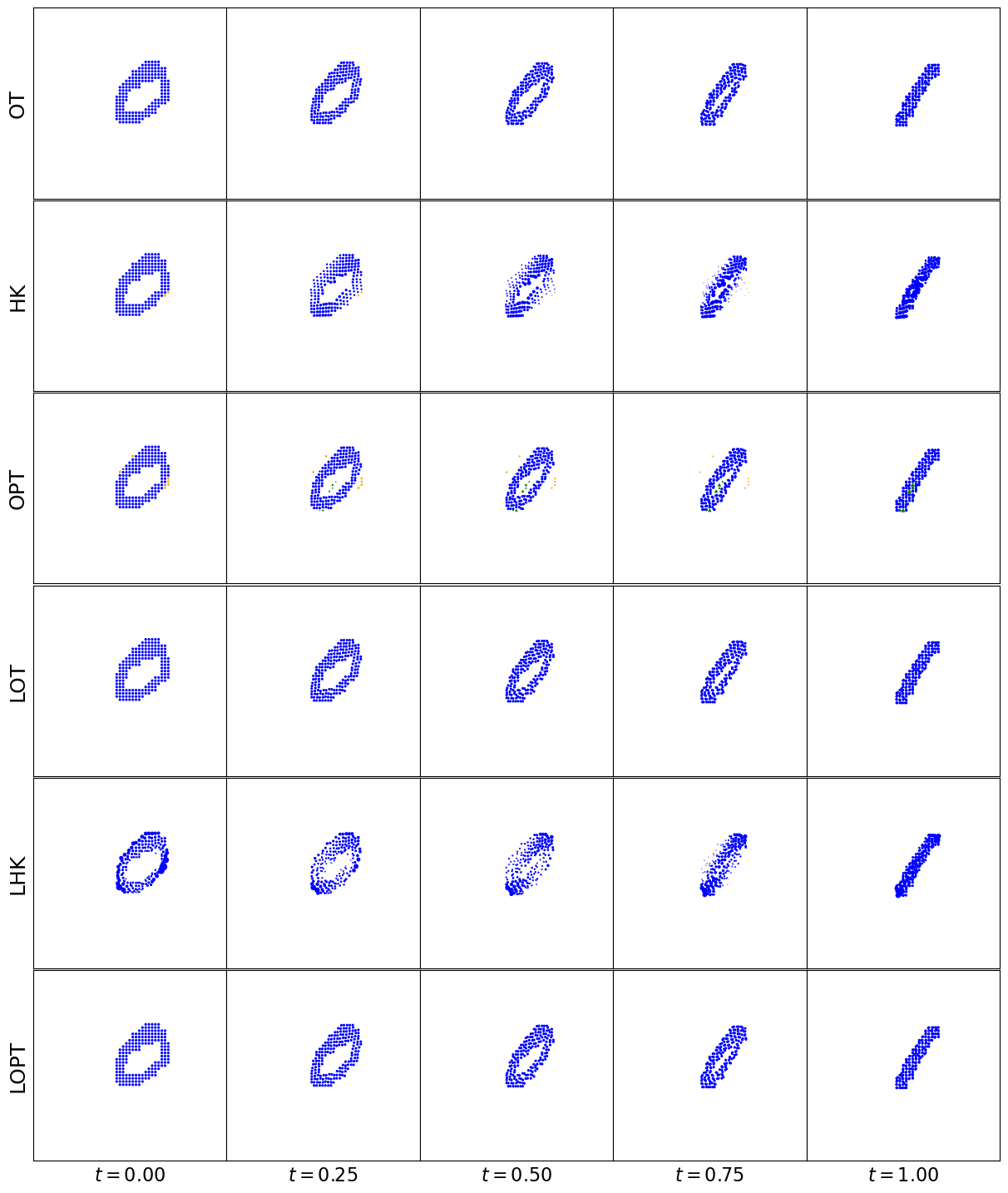} 
       \caption{$\eta=0$}
        \label{fig: 01_0}
    \end{subfigure}
    \begin{subfigure}[b]{0.33\textwidth}
        \centering
        \includegraphics[width=1\textwidth]{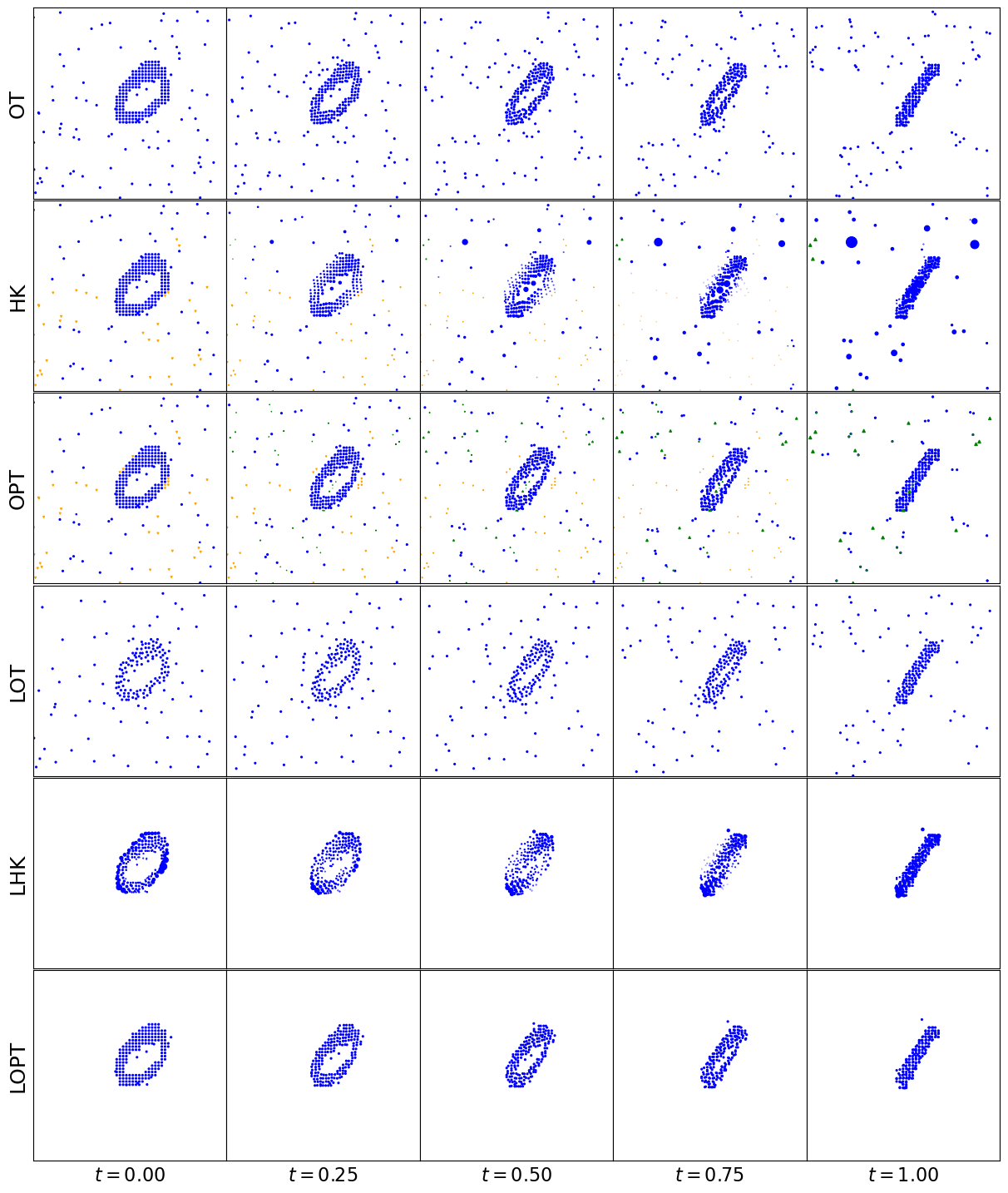}
        \caption{$\eta=0.5$}
        \label{fig: 01_0.5}
    \end{subfigure}
    \begin{subfigure}[b]{0.33\textwidth}
        \centering
        \includegraphics[width=1\textwidth]{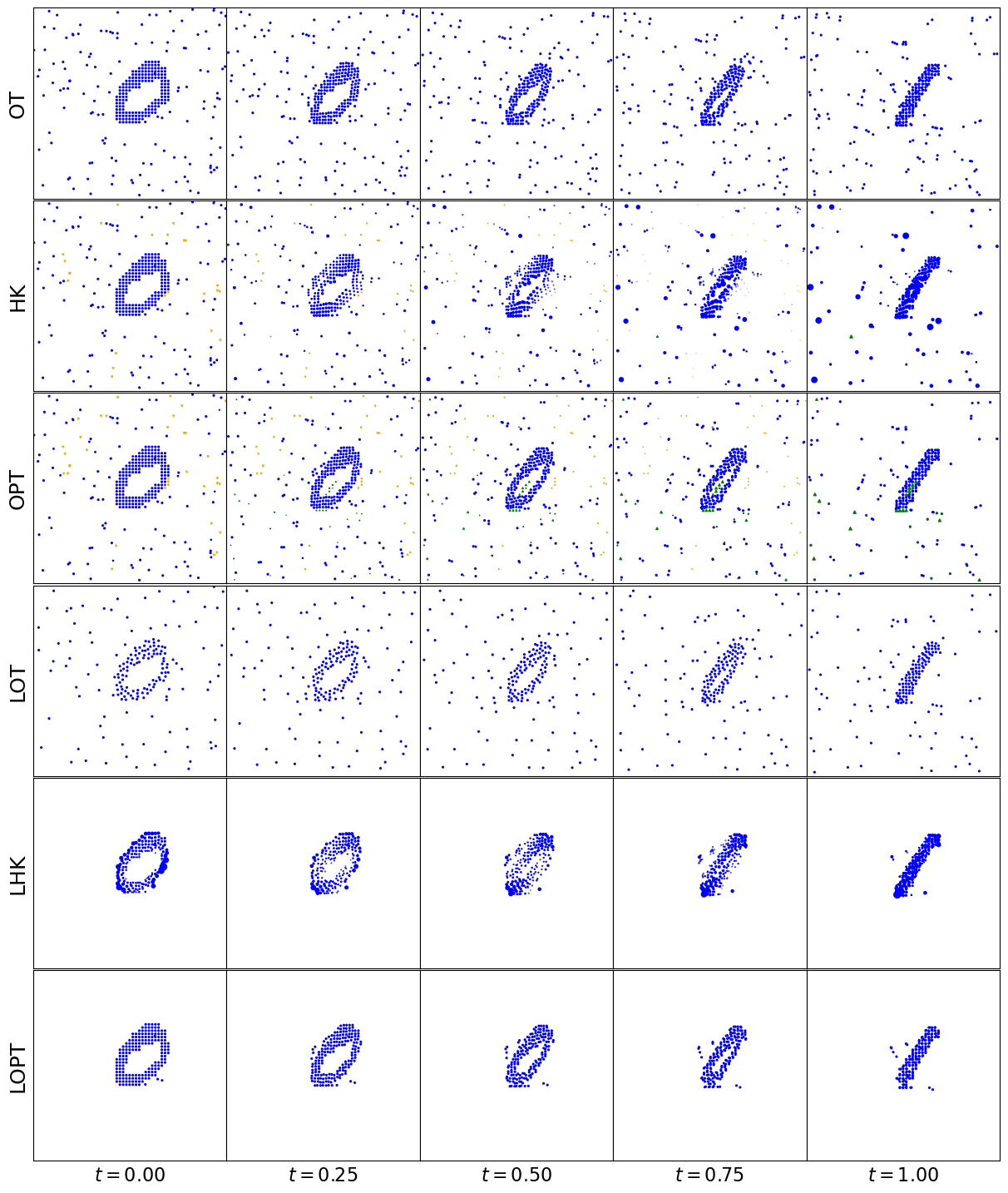}
        \caption{$\eta=0.75$}
        \label{fig: 01_0.75}
    \end{subfigure}
    
    \caption{Interpolation between two point-clouds at different times $t\in\{0, 0.25, 0.5,0.75,1\}$. Different values of noise $\eta$ were considered for the different interpolation approaches (OT, LOP, OPT, LOPT). For LOT and LOPT the reference measure is the barycenter between the PointClouds $0$, $1$, and $9$ with no noise. Top row: Interpolation between digits $0$ and $9$. Bottom row: Interpolation between digits $0$ and $1$.}
    \label{fig: interpolation_09_01}
\end{figure}

\subsection{Preliminary comparisons between LHK and LOPT}\label{sec: HK vs OPT}

The contrasts between the Linearized Hellinger-Kantorovich (LHK) \cite{cai2022linearized} and the LOPT approaches come from the differences between HK (Hellinger-Kantorovich) and OPT distances. 

The main issue is that for HK transportation between two measures, the transported portion of the mass does not resemble the OT-geodesic where mass is preserved. In other words, HK changes the mass as it transports it, while OPT preserves it. 

This issue is depicted in Figure \ref{fig: HK and OPT}. In that figure, both the initial (blue) and final (purple) distributions of masses are two unit-mass delta measures at different locations. Mass decreases and then increases while being transported for HK, while it remains constant for OPT. For HK, the transported portion of the mass does not resemble the OT-geodesic where mass is preserved. In other words, HK changes the mass as it transports it, while OPT preserves it. 

To illustrate better this point, we incorporate here Figure \ref{fig: Dirac_HK_OPT} which is the 
two-dimensional analog of Figure \ref{fig: HK and OPT}.

\begin{figure}
    \centering
    \includegraphics[width=0.75\textwidth]{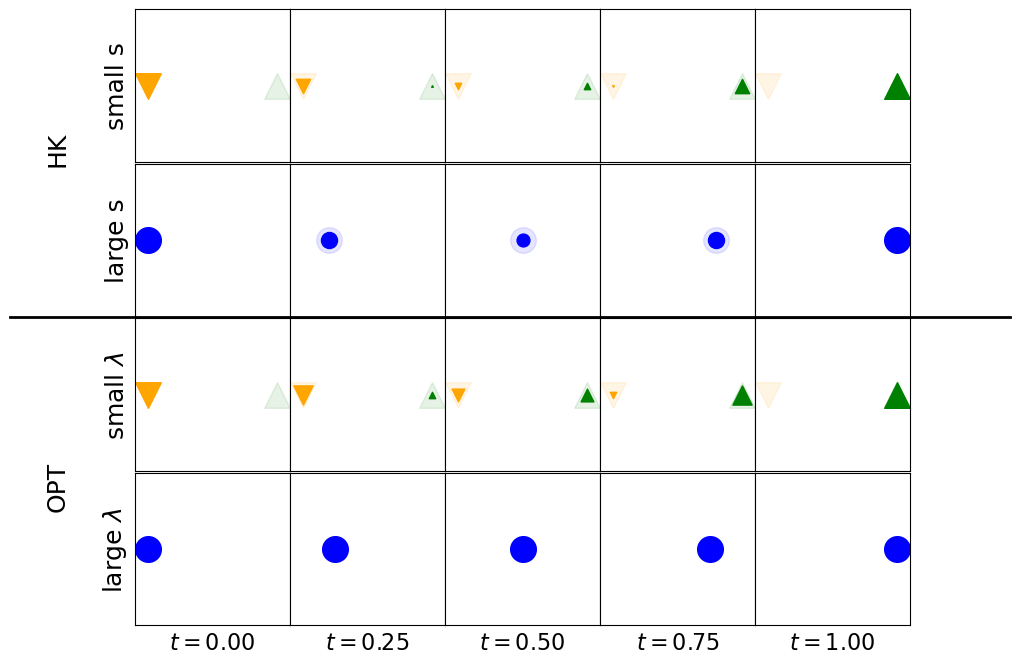}
    \caption{HK vs. OPT interpolation between two delta measures of unit mass located at different positions. Two scenarios are shown for each transport framework varying the respective parameters ($s$ for HK, and $\lambda$ for OPT).  Blue circles stand for the cases when the geodesic transports mass. Triangular shapes represent destruction and creation of mass. Triangles pointing down in orange indicate the locations where mass will be destroyed. Triangles pointing up in green indicate the locations where mass will be created. On top of that shadows are added to emphasize the change of mass from initial and final configurations. The Top two plots exhibit two extreme cases when performing HK geodesic. When $s$ is small, everything is created/destroyed, when s is large everything is transported \textit{without mass-preservation}. On the other side, the Bottom two plots show the two analogous cases for OPT geodesics. When $\lambda$ is small we observe creation/destruction, when $\lambda$ is large we have \textit{mass-preserving transportation}. The intermediate cases are treated in the following.}
    \label{fig: Dirac_HK_OPT}
\end{figure}

In addition, in Figure \ref{fig: trig config} we not only compare HK and OPT, but also LHK and LOPT. The top subfigure \ref{fig: trig config a} shows the measures $\mu_1$ (blue dots) and $\mu_2$ (orange crosses) to be interpolated with the different techniques in the next subfigures \ref{fig: trig config b} and \ref{fig: trig config c}. Also, in \ref{fig: trig config a} we plot the reference $\mu_0$ (green triangles) which is going to be used only in experiment \ref{fig: trig config c}. 
The measures $\mu_1$ and $\mu_2$ are interpreted as follows. The three masses on the interior, forming a triangle, can be considered as the signal information and the two masses on the corners can be considered noise. That is, we can assume they are just two noisy point cloud representations of the same distribution shifted. We want to transport the three masses in the middle without affecting their mass and relative positions too much. 
Subfigure \ref{fig: trig config b} shows HK and OPT interpolation between the two measures $\mu_1$ and $\mu_2$. In the HK cases, we notice not only how the masses vary, but also how their relative positions change obtaining a very different configuration at the end of the interpolation. OPT instead returns a more natural interpolation because of the mass preservation of the transported portion and the decoupling between transport and destruction/creation of mass. 
Finally, subfigure \ref{fig: trig config c} shows the interpolation of $\mu_1$ and $\mu_2$ for LHK and LOPT. The reference measure $\mu_0$ is the same for both and we can see the \textit{denoising} effect due to the fact that, for the three measures, the mass is concentrated in the triangular region in the center. However, while mass and structure-preserving transportation can be seen for LOPT, for LHK the shape of the configuration changes.

On top of that, as is often the case on quantization, point cloud representations of measures are given as samples with uniform mass. OPT/LOPT interpolation will not change the mass of each transported point. Therefore, the intermediate steps of an algorithm using OPT/LOPT transport benefit from conserving the same type of representation. That is, as a uniform set of points.

\begin{figure}
\centering
\begin{subfigure}{0.75\textwidth}
\centering
\includegraphics[width=0.5\textwidth]{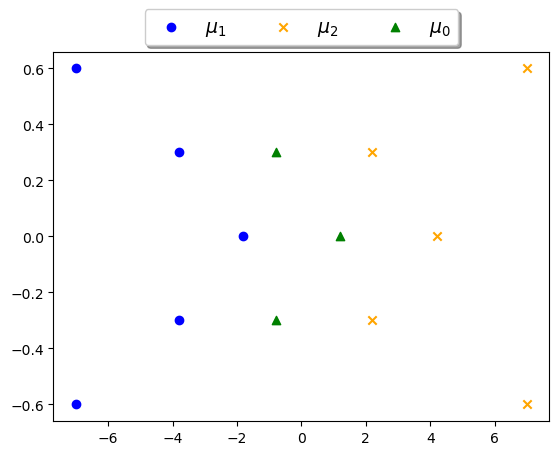}
        \caption{Plot of two measures ($\mu_1$ as blue circles and $\mu_2$ as orange crosses, where each point has mass one), and a reference measure $\mu_0$ (concentrated on the three green triangular locations, unit mass each).}
        \label{fig: trig config a} 
\end{subfigure}
\begin{subfigure}{0.75\textwidth}
        \centering
\includegraphics[width=0.8\textwidth]{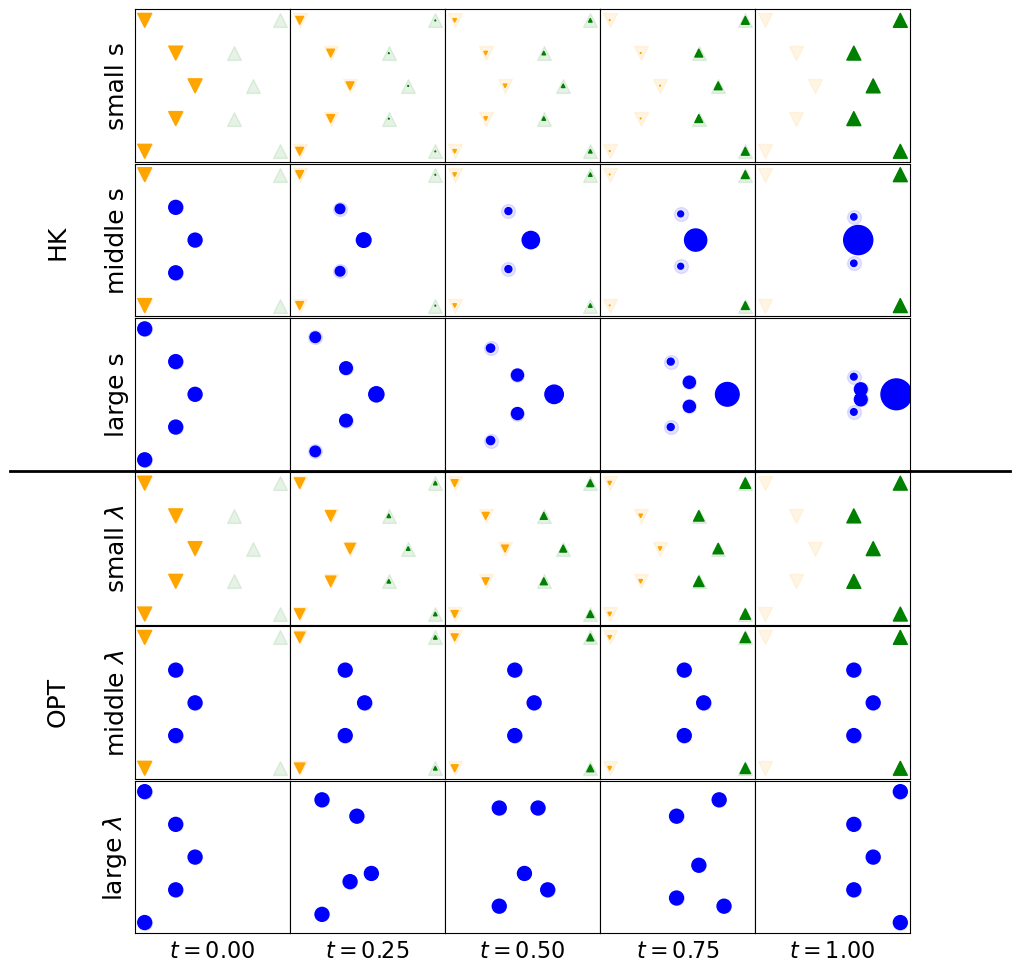}
        \caption{HK and OPT interpolation between the two measures $\mu_1$ and $\mu_2$ at different times. The color code is the same as for Figure \ref{fig: Dirac_HK_OPT}.}
        \label{fig: trig config b} 
    \end{subfigure}\hspace{1em}
\begin{subfigure}{0.75\textwidth}
\centering
\includegraphics[width=0.8\textwidth]{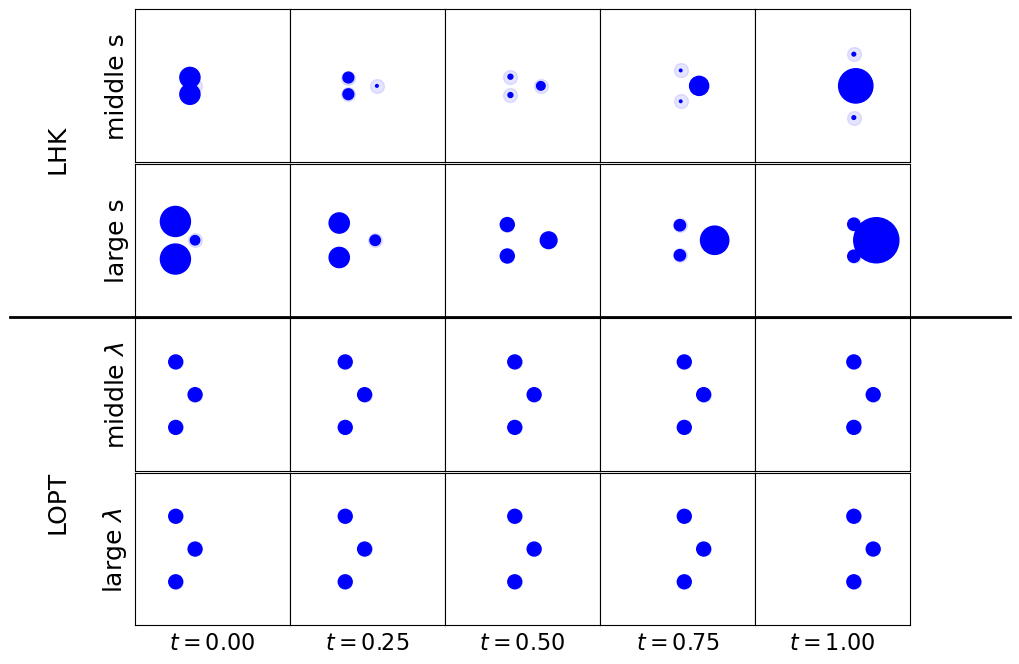}
        \caption{LHK and LOPT interpolation between the two measures $\mu_1$ and $\mu_2$ at different times using for both cases the same reference $\mu_0$.}
        \label{fig: trig config c} 
    \end{subfigure}   
    \caption{HK vs OPT and LHK vs LOPT}
        \label{fig: trig config}  
\end{figure}

For comparison on PointCloud interpolation using MNIST, we include Figure \ref{fig: all comparisons} that illustrates OPT, LOPT, HK, and LHK interpolation for digit pairs (0,1) and (0,9). In the visualizations, the size of each circle is plotted according to amount of the mass at each location.

\begin{figure}
    \centering
    \begin{subfigure}[b]{1\textwidth}
        \centering
        \includegraphics[width=0.3\textwidth]{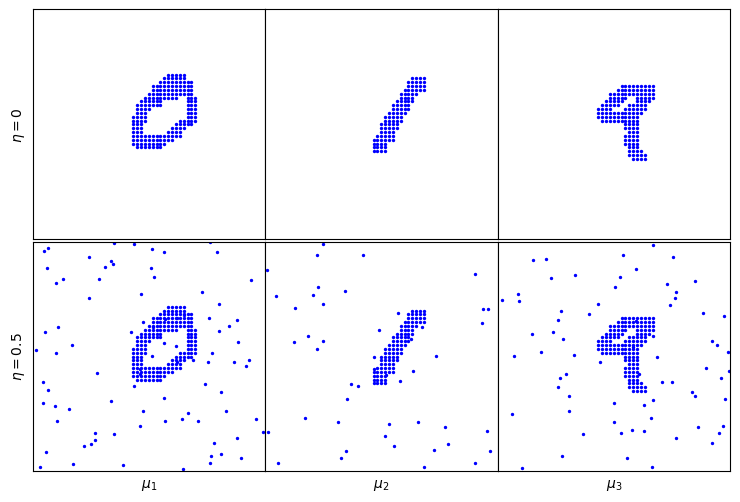}
        \caption{Samples of digits 0, 1, and 9 from MNIST Data Set. Top row: clean point cloud. Bottom row: $50\%$ of noise.}
        \label{fig: data 0 1 9}        
    \end{subfigure}
        \begin{subfigure}[b]{0.49\textwidth}
        \centering
        \includegraphics[width=1\textwidth]{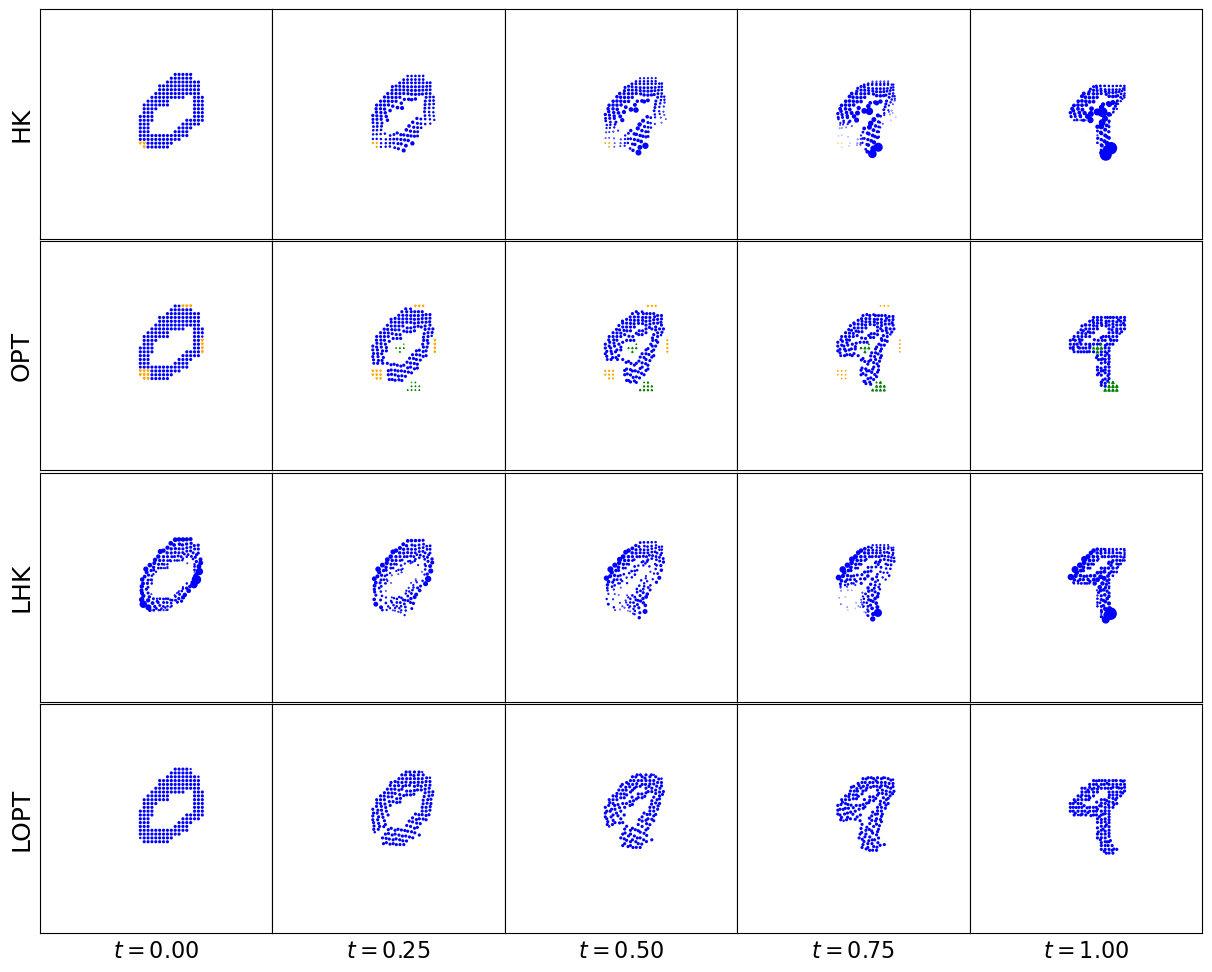}
        \caption{Noise level $\eta=0$.}
        \label{fig: 0-9.clean.weight}        
    \end{subfigure}
        \begin{subfigure}[b]{0.49\textwidth}
        \centering
        \includegraphics[width=1\textwidth]{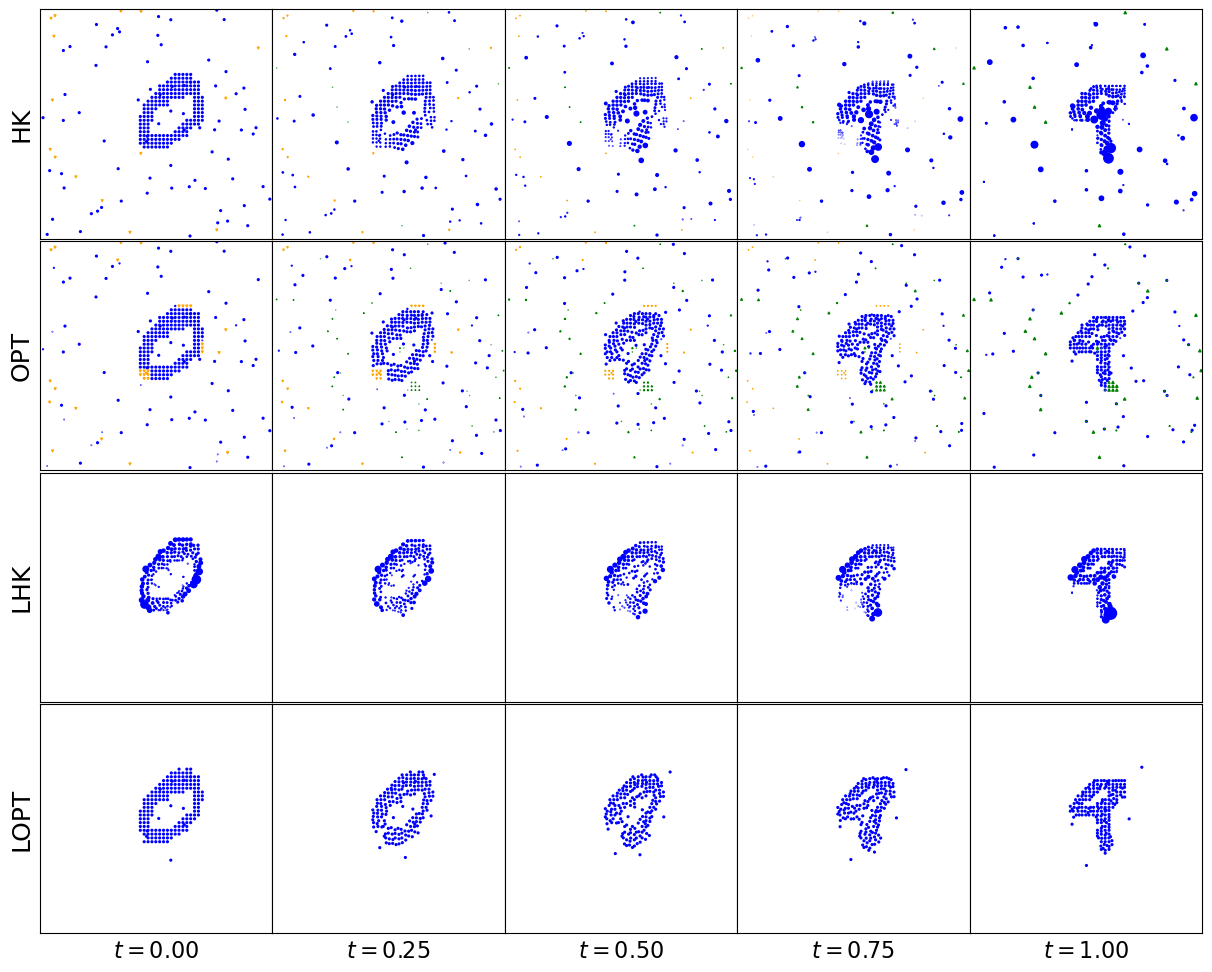}
        \caption{Noise level $\eta=0.5$}
        \label{fig: 0-9.noise.weight}        
    \end{subfigure}
    \begin{subfigure}[b]{0.49\textwidth}
        \centering
        \includegraphics[width=1\textwidth]{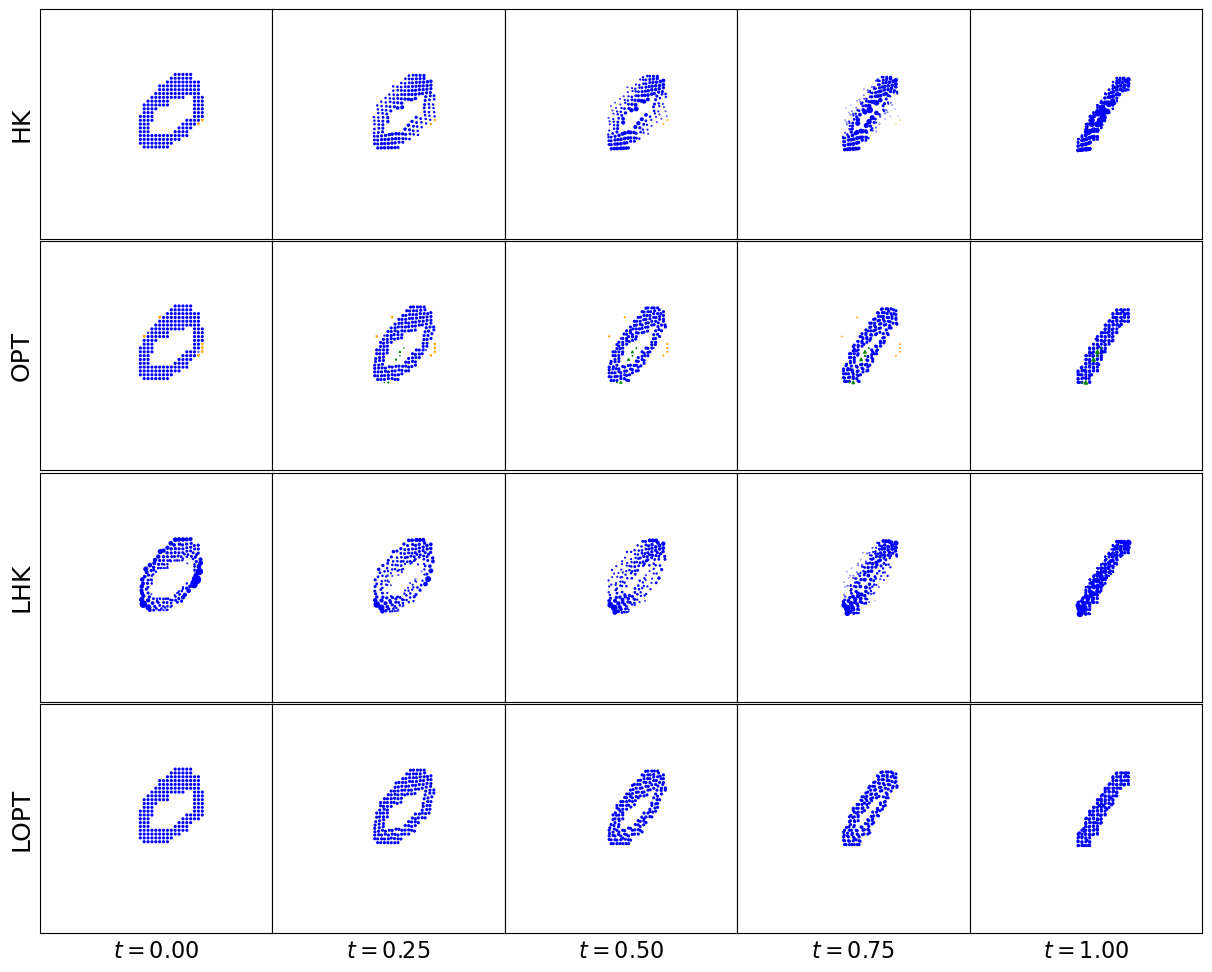}
        \caption{Noise level $\eta=0.$}
        \label{fig: 0-1.clean.weight}        
    \end{subfigure}
        \begin{subfigure}[b]{0.49\textwidth}
        \centering
        \includegraphics[width=1\textwidth]{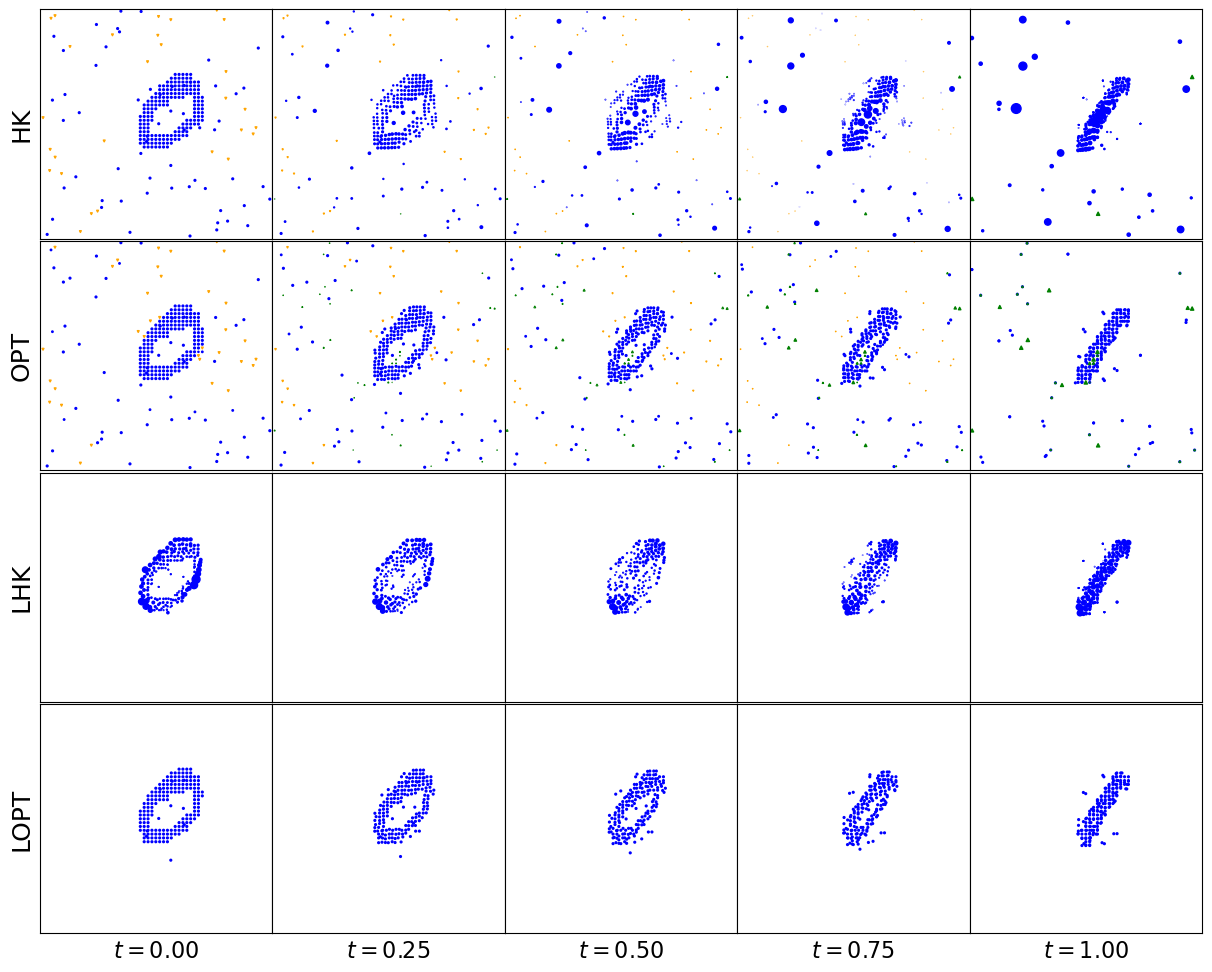}
        \caption{Noise level $\eta=0.5$.}
        \label{fig: 0-1.noise.weight}        
    \end{subfigure}
    \caption{Point cloud interpolation using all the techniques unbalanced transportation HK, OPT, LHK, and LOPT.}
    \label{fig: all comparisons}
\end{figure}

However, we do not claim the presented LOPT tool to be a one size fits all kind of tool.
We are working on the subtle differences between LHK and LOPT and expect to have a complete and clear picture in the future. The aim of this article was to present a new tool with an intuitive introduction and motivation so that the OT community would benefit from it.

\subsection{Barycenter computation}

Can the linear embedding technique be used to compute the barycenter of a set of measures, e.g., by computing the barycenter of the embeddings and performing an inversion to the original space?

One can first calculate the mean of the embedded measures, and recover a measure from this mean. The recovered measure is not necessarily the OPT barycenter, however, one can repeat this process and obtain better approximations of the barycenter. Similar to LOT, we have numerically observed that such a process will converge to a measure that is close to the barycenter. However, there are several technical considerations that one needs to pay close attention. For instance, the choice of $\lambda$ and the choice of the initial reference measure are critical in this process.

\begin{figure}
    \centering
    \includegraphics[width=1\textwidth]{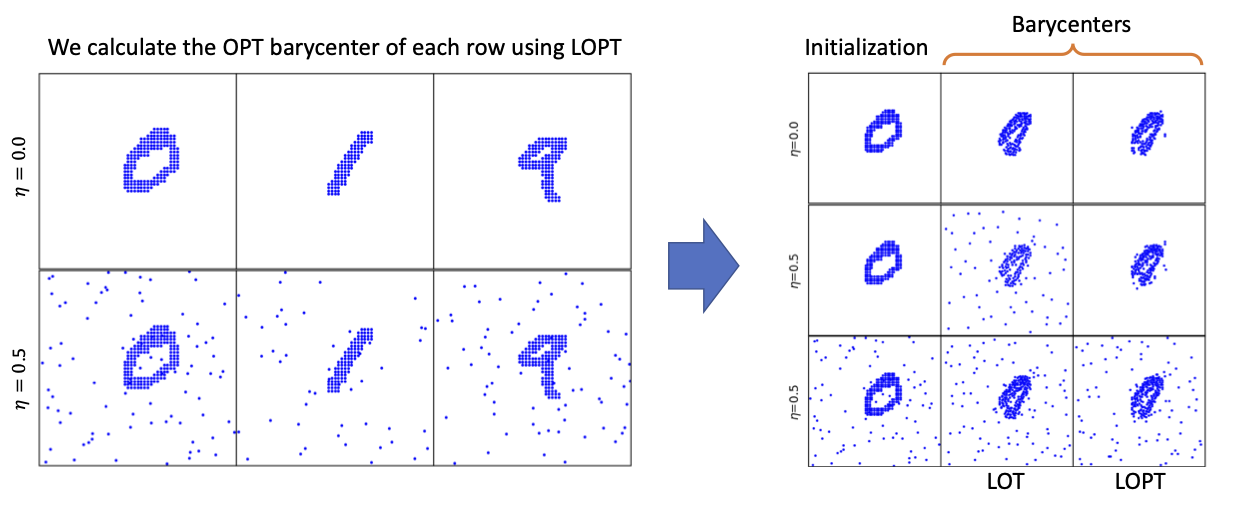}
    \caption{
    The depiction of barycenters between digits $0$ and $1$, and between $0$ and $9$ using the LOPT technique. 
    Left panel: original measures (point-clouds) $\mu_1$ (digit $0$), $\mu_2$ (digit $1$), and $\mu_3$ (digit $9$). We considered them with no noise on the top left panel ($\eta=0$), and with corrupted under noise on the bottom left panel (level of noise $\eta=0.5$). 
    Right panel: The first row is the result that both initial measure and data $\mu_1,\mu_2,\mu_3$ are clean data; the second row is the result of clean initial measure and noise corrupted data; the third row is the result for noise corrupted initial measure and data.}
    \label{fig: baricenter_new_computation}
\end{figure}

\end{document}